%% file: main.tex
\RequirePackage{fix-cm}
\documentclass[twocolumn]{svjour3}          
\smartqed  
\usepackage{times}
\usepackage[utf8]{inputenc}
\usepackage[sort&compress,numbers]{natbib}
\usepackage{multicol}
\usepackage[ruled,vlined,linesnumbered]{algorithm2e}
\usepackage[noend]{algpseudocode}

\usepackage{amsfonts}
\usepackage{amssymb,amsmath}
\usepackage{graphicx}
\usepackage{subcaption}
\usepackage[font=footnotesize]{subcaption}
\usepackage[font=footnotesize]{caption}
\usepackage{color}
\usepackage{booktabs} 
\usepackage{tabulary}
\usepackage{lipsum}

\usepackage{amsthm} 
\usepackage{thmtools}
\usepackage{mathtools}

\usepackage{hyperref}
\usepackage[usenames,dvipsnames,svgnames,table]{xcolor}
\usepackage{colortbl}
\definecolor{lightgreen}{rgb}{0.8,1,0.8}
\newcommand{\cmark}{\cellcolor{lightgreen}}%

\hypersetup{bookmarksopen,bookmarksnumbered,
pdfpagemode=UseOutlines,
colorlinks=true,
linkcolor=blue,
anchorcolor=blue,
citecolor=blue,
filecolor=blue,
menucolor=blue,
urlcolor=blue
}

\graphicspath{{./}}

\input{math_defns}

\input{math_operators}

\input{text_shortcuts}

\newcommand{\eref}[1]{(\ref{#1})}
\newcommand{\sref}[1]{Section~\ref{#1}}
\newcommand{\figref}[1]{Fig.~\ref{#1}}
\newcommand{\algoref}[1]{Algorithm~\ref{#1}}
\newcommand{\algolineref}[1]{Line~\ref{#1}}

\newtheoremstyle{hypstyle}
{3pt} 
{3pt} 
{\itshape} 
{} 
{\bfseries} 
{.} 
{.5em} 
{} 

\theoremstyle{hypstyle} 
\newtheorem{observation}{O}
\newtheorem{ques}{Q}

\newcommand{\xxnote}[3]{}
\ifx\hidenotes\undefined
  \renewcommand{\xxnote}[3]{\color{#2}{#1: #3}}
\fi

\newcommand{\fullFigGap}[0]{\vspace{-1.5\baselineskip}} 

\begin{document}

\title{Leveraging Experience in Lazy Search\thanks{This work was (partially) funded by the National Institute of Health R01 (\#R01EB019335), National Science Foundation CPS (\#1544797), National Science Foundation NRI (\#1637748), National Science Foundation CAREER (\#1750483), the Office of Naval Research, the RCTA, Amazon, and Honda Research Institute USA.}
}

\author{Mohak Bhardwaj \and
Sanjiban Choudhury \and
Byron Boots \and 
Siddhartha Srinivasa \\
}
\institute{Mohak Bhardwaj \at
           University of Washington \\
           Bill \& Melinda Gates Center \\
           3800 E Stevens Way NE \\
           Seattle, WA 98195 \\
           Tel.: +1(717) 788-9103\\
           Fax: +1(206) 543-2969\\
           \email{mohakb@uw.edu}           
           \and
           Sanjiban Choudhury \at
           Aurora Innovation, Inc. \\
           113 47th St \\
           Pittsburgh, PA 15201
           \and
           Byron Boots \at
           University of Washington \\
           Bill \& Melinda Gates Center \\
           3800 E Stevens Way NE \\
           Seattle, WA 98195
           \and 
           Siddhartha Srinivasa \at
           University of Washington \\
           Bill \& Melinda Gates Center \\
           3800 E Stevens Way NE \\
           Seattle, WA 98195
}

\maketitle
\vspace{-4mm}
\input{abstract.tex}

\input{introduction.tex}

\input{problem_formulation.tex}

\input{challenge.tex}

\input{approach.tex}
\input{analysis}
\input{bayesian}
\input{experiments.tex}

\input{related_work.tex}

\input{discussions.tex}

\bibliographystyle{unsrtnat}
\bibliography{references}
\newpage

\end{document}

%% file: math_defns.tex
\newcommand{\explicitGraph}[0]{G}
\newcommand{\vertexSet}[0]{V}
\newcommand{\edgeSet}[0]{E}
\newcommand{\vertex}[0]{v}
\newcommand{\edge}[0]{e}
\newcommand{\start}[0]{v_s}
\newcommand{\goal}[0]{v_g}
\newcommand{\cost}[0]{c}

\newcommand{\length}[0]{\ell}

\renewcommand{\path}[0]{\xi}

\newcommand{\pathSet}[0]{\Xi}

\newcommand{\world}[0]{\phi}

\newcommand{\evalEdges}[0]{E_{\mathrm{eval}}}

\newcommand{\edgesValid}[0]{E_{\mathrm{val}}}
\newcommand{\edgesInvalid}[0]{E_{\mathrm{inv}}}
\newcommand{\vectorp}[0]{\mathbf{p}}

\newcommand{\stateSpace}{\mathcal{S}}
\newcommand{\actionSpace}{\mathcal{A}}
\newcommand{\transFn}{\mathrm{T}}
\newcommand{\rewardFn}{\mathrm{R}}
\newcommand{\state}{s}
\newcommand{\action}{a}

\newcommand{\discount}{\gamma}
\newcommand{\stateAbs}{\mathcal{G}}

\newcommand{\valueFnPolicy}[1]{V^{#1}}

\newcommand{\QFn}[1]{Q^{#1}}

\newcommand{\dataset}{\mathcal{D}}
\newcommand{\mixParam}{\beta}

\newcommand{\policyLEARN}[0]{\hat{\pi}}

\newcommand{\policyOR}[0]{\pi_{\mathrm{OR}}}
\newcommand{\policyAOR}[0]{\pi_{\mathrm{AOR}}}

\newcommand{\policyORBel}[0]{\tilde{\pi}_{\mathrm{OR}}}
\newcommand{\policyAORBel}[0]{\tilde{\pi}_{\mathrm{AOR}}}

\newcommand{\policy}[0]{\pi}
\newcommand{\policySpace}[0]{\Pi}
\newcommand{\policyOpt}[0]{\policy^*}
\newcommand{\policyLearn}[0]{\hat{\policy}}
\newcommand{\policyMix}[0]{\policy_\mathrm{mix}}
\newcommand{\policyOracle}[0]{\policy_\mathrm{OR}}

\newcommand{\policyRoll}[0]{\pi_{\mathrm{roll}}}

\newcommand{\hyp}[0]{h}
\newcommand{\hypSet}[0]{\mathcal{H}}
\newcommand{\numHyp}[0]{n}
\newcommand{\test}[0]{t}
\newcommand{\testSet}[0]{\mathcal{T}}
\newcommand{\numTest}[0]{l}
\newcommand{\outcome}[0]{x}

\newcommand{\region}[0]{\mathcal{R}}
\newcommand{\numRegion}[0]{m}

\newcommand{\versionSpace}[0]{\hypSet}
\newcommand{\selTestSet}[0]{\mathcal{A}}
\newcommand{\outcomeSet}[1]{\mathbf{\outcome}_{#1}}
\newcommand{\obsOutcome}[0]{\outcomeSet{\selTestSet}}
\newcommand{\obsOutcomeFunc}[2]{\obsOutcome \left( #1, #2\right)}
\newcommand{\outcomeTest}[1]{\outcome_{#1}}
\newcommand{\obsOutcomeAdd}[1]{\outcomeSet{\selTestSet \cup \set{#1}}}

\newcommand{\graphDRD}[0]{\mathcal{G}}
\newcommand{\vertexSetDRD}[0]{\mathcal{V}}
\newcommand{\edgeSetDRD}[0]{\mathcal{E}}
\newcommand{\weight}[0]{w}
\newcommand{\fec}[1]{f_{\mathrm{EC}}\left(#1\right)}
\newcommand{\gain}[2]{\Delta \left( #1 \;|\; #2 \right)}
\newcommand{\policyEC}[0]{\pi_{\mathrm{EC}}}

%% file: math_operators.tex
\DeclareMathOperator*{\argmin}{arg\,min}
\DeclareMathOperator*{\argmax}{arg\,max}

\newcommand{\maxprob}[1]{\underset{#1}{\max}}
\newcommand{\minprob}[1]{\underset{#1}{\min}}

\newcommand{\argmaxprob}[1]{\underset{#1}{\argmax}}
\newcommand{\argminprob}[1]{\underset{#1}{\argmin}}
\newcommand{\norm}[2]{\left|\left| #1 \right|\right|_{#2}}
\newcommand{\abs}[1]{\left|#1\right|}
\newcommand{\card}[1]{\left|#1\right|}

\DeclareMathOperator*{\suchthat}{\;\; \mbox{s.t.} \;\;}
\newcommand{\expect}[2]{\mathbb{E}_{#1}\left[#2\right]}

\newcommand{\real}[0]{\mathbb{R}}

\newcommand{\bbm}{\begin{bmatrix}}
\newcommand{\ebm}{\end{bmatrix}}

\newcommand{\pair}[2]{\left( #1, #2\right)}
\newcommand{\set}[1]{\left\lbrace #1\right\rbrace}
\newcommand{\seq}[2]{\left(#1_{1}, #1_{2}, \ldots, #1_{#2}\right)}

\newcommand{\setst}[2]{\left\lbrace #1\;\middle|\;#2\right\rbrace}

\newcommand{\Ind}[0]{\mathbb{I}}

%% file: text_shortcuts.tex
\newcommand{\daggerAlg}[0]{\textsc{DAgger}}
\newcommand{\supervisedAlg}[0]{\textsc{Supervised}}

\newcommand{\algName}[0]{\textsc{StrOLL}\xspace}
\newcommand{\algFullName}[0]{\textbf{S}earch \textbf{t}h\textbf{r}ough \textbf{O}racle \textbf{L}earning and \textbf{L}aziness \xspace}
\newcommand{\algHeuristic}[0]{\textsc{StrOLL-R}\xspace}

\newcommand{\bisect}[0]{\textsc{BiSECt}\xspace}
\newcommand{\direct}[0]{\textsc{DiRECt}\xspace}

\newcommand{\ecsq}[0]{EC\textsuperscript{2}\xspace}
\newcommand{\lazysp}[0]{\textsc{LazySP}\xspace}

\newcommand{\selector}[0]{\textsc{Selector}\xspace}

\newcommand{\selectorOracle}[0]{\textsc{Oracle}\xspace}

\newcommand{\selectorRandom}[0]{\textsc{Random}\xspace}
\newcommand{\selectorForward}[0]{\textsc{Forward}\xspace}
\newcommand{\selectorBackward}[0]{\textsc{Backward}\xspace}
\newcommand{\selectorAlternate}[0]{\textsc{Alternate}\xspace}
\newcommand{\selectorFailFast}[0]{\textsc{FailFast}\xspace}
\newcommand{\selectorPostFailFast}[0]{\textsc{PostFailFast}\xspace}

\newcommand{\featPost}[0]{\textsc{Posterior}\xspace}
\newcommand{\featPrior}[0]{\textsc{Prior}\xspace}
\newcommand{\featIndex}[0]{\textsc{Location}\xspace}
\newcommand{\featDeltaLength}[0]{$\Delta$-\textsc{Length}\xspace}
\newcommand{\featDeltaEval}[0]{$\Delta$-\textsc{Eval}\xspace}
\newcommand{\featPDL}[0]{P\featDeltaLength}

\newcommand{\selectorPDL}[0]{P\featDeltaLength\xspace}

%% file: abstract.tex
\begin{abstract}
Lazy graph search algorithms are efficient at solving motion planning problems where edge evaluation is the computational bottleneck. 
These algorithms work by lazily computing the shortest potentially feasible path, evaluating edges along that path, and repeating until a feasible path is found. 
The order in which edges are selected is critical to minimizing the total number of edge evaluations: 
a good edge selector chooses edges that are not only likely to be invalid, but also eliminates future paths from consideration.
We wish to learn such a selector by leveraging prior experience. 
We formulate this problem as a Markov Decision Process (MDP) on the state of the search problem. 
While solving this large MDP is generally intractable,
 we show that 
 we can compute oracular selectors that can solve the MDP during training.
With access to such oracles, we  use imitation learning to find effective policies. If new search problems are sufficiently similar to problems solved during training, the learned policy will choose a good edge evaluation ordering and solve the motion planning problem quickly. 
We evaluate our algorithms on a wide range of $2$D and $7$D problems and show that the learned selector outperforms baseline commonly used heuristics. 
We further provide a novel theoretical analysis of lazy search in a Bayesian framework as well as regret guarantees on our imitation learning based approach to motion planning.

\keywords{Imitation Learning \and Motion Planning \and Lazy Search}

\end{abstract}

%% file: introduction.tex

\section{Introduction}
\label{sec:introduction}

In this paper, we explore algorithms that leverage past experience to find the shortest path on a graph while minimizing planning time.  
We focus on the domain of robot motion planning where the planning time is dominated by \emph{edge evaluation}~\citep{hauser2015lazy}.
Here the goal is to check the minimal number of edges, invalidating potential shortest paths along the way, until we discover the shortest feasible path -- this is the central tenet of lazy search~\citep{dellin2016unifying,bohlin2000path}. 
We propose to \emph{learn within this framework} which edges to evaluate (\figref{fig:intro}).

How should we leverage experience? Consider the ``Piano Mover's Problem''~\citep{schwartz1983piano} where the goal is to plan a path for a piano from one room in a house to another. Collision checking all possible motions of the piano can be quite time-consuming. Instead, what can we infer if we were given a database of houses and edge evaluations results?
\begin{enumerate}
	\item \emph{Check doors first} - these edges serve as bottlenecks for many paths which can be eliminated early if invalid.
	\item \emph{Prioritize narrow doors} - these edges are more likely to be invalid and can save checking other edges.
	\item \emph{Similar doors, similar outcomes} - these edges are correlated, checking one reveals information about others.
\end{enumerate}
Intuitively, we need to consider all past discoveries about edges to make a decision. While this has been explored in the Bayesian setting~\citep{choudhury2017active,choudhury2018bayesian}, we show that more generally the problem can be mapped to a Markov Decision Process (MDP).
However, the size of the MDP grows exponentially with the size of the graph. Even if we were to use approximate dynamic programming, we still need to explore an inordinate number of states to learn a reasonable policy.  

Interestingly, if we were to reveal the status of all the edges during training, we can conceive of a \emph{clairvoyant oracle}~\citep{choudhury2017data} that can select the optimal sequence of edges to invalidate. In fact, we show that the oracular selector is equivalent to \emph{set cover}, for which greedy approximations exist. By imitating clairvoyant oracles~\citep{choudhury2017data}, we can drastically cut down on exploration and focus learning on a small, relevant portion of the state space~\cite{sun2017deeply}. This leads to a key insight: use imitation learning to quickly bootstrap the selector to match oracular performance.  
%
We propose a new algorithm, \algName, that deploys an interactive imitation learning framework~\citep{ross2011reduction} to train the edge selector (\figref{fig:illustration_algorithm}). At every iteration, it samples a world (validity status for all edges) and executes the learner. At every timestep, it queries the clairvoyant oracle associated with the world to select an edge to evaluate. This can be viewed as a classification problem where the goal is to map features extracted from edges to the edge selected by the oracle. This datapoint is aggregated with past data, which is then used to update the learner. 

We also theoretically analyze the problem in the Bayesian setting to characterize lower bounds. We show that the problem can be mapped to an instance of Bayesian Active Learning, and hence is NP-Hard. However, this mapping allows us to harness the theory of adaptive submodularity to derive analytic, near-optimal policies for edge evaluation. To the best of our knowledge, this is the first such bound on lazy shortest path. 

In summary, our main contributions are:
\begin{enumerate}
	\item We map edge selection in lazy search to an MDP (\sref{sec:problem_formulation}) and solve it for small graphs (\sref{sec:challenge}).
	\item We show that larger MDPs, can be efficiently solved by imitating clairvoyant oracles (\sref{sec:approach}).
	\item We show that the learned policy  can outperform competitive baselines on a wide range of datasets (\sref{sec:experiments}).
	\item We also derive an analytic policy that achieves near Bayes-optimality, the first such bounds for lazy shortest path (\sref{sec:bayesian_sp}).
\end{enumerate}

This paper is an extension of~\citep{bhardwaj2019leveraging}. We introduce the following new content
\begin{enumerate}
	\item In \sref{sec:bayesian_sp}, we introduce the paradigm of Bayesian Lazy Shortest Path and present a theoretical analysis that maps it to the Bayesian Decision Region Determination (DRD) problem for which near-optimal policies can be derived.
	\item In \sref{sec:analysis}, we provide theoretical bounds on the performance \algName via a regret analysis that exploits a connection between imitation learning of clairvoyant oracles to hindsight optimization. We also show that in the special case of independent Bernoulli edges, the optimal edge selector lies within the policy class of \algName. 
	\item We provide additional experimental details for \algName and a qualitative comparison that sheds light on the efficacy of the learned selector versus uninformed heuristics commonly used in practice.
\end{enumerate}

\begin{figure}[!t]
\centering
\includegraphics[width=\columnwidth]{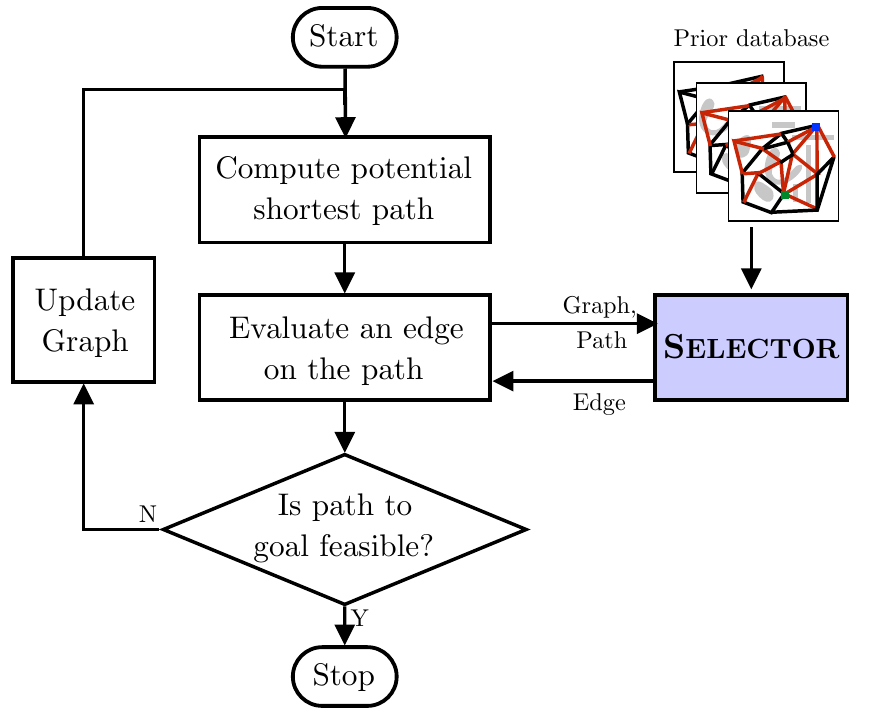}
\caption{The \lazysp~\citep{dellin2016unifying} framework. \lazysp iteratively computes the shortest path, queries a \selector for an edge on the path, evaluates it and updates the graph until a feasible path is found. The number of edges evaluated depends on the choice of \selector. We propose to train a \selector from prior data. \fullFigGap}
\label{fig:intro}
\end{figure}

\begin{figure*}[!t]
    \centering
    \includegraphics[width=\textwidth]{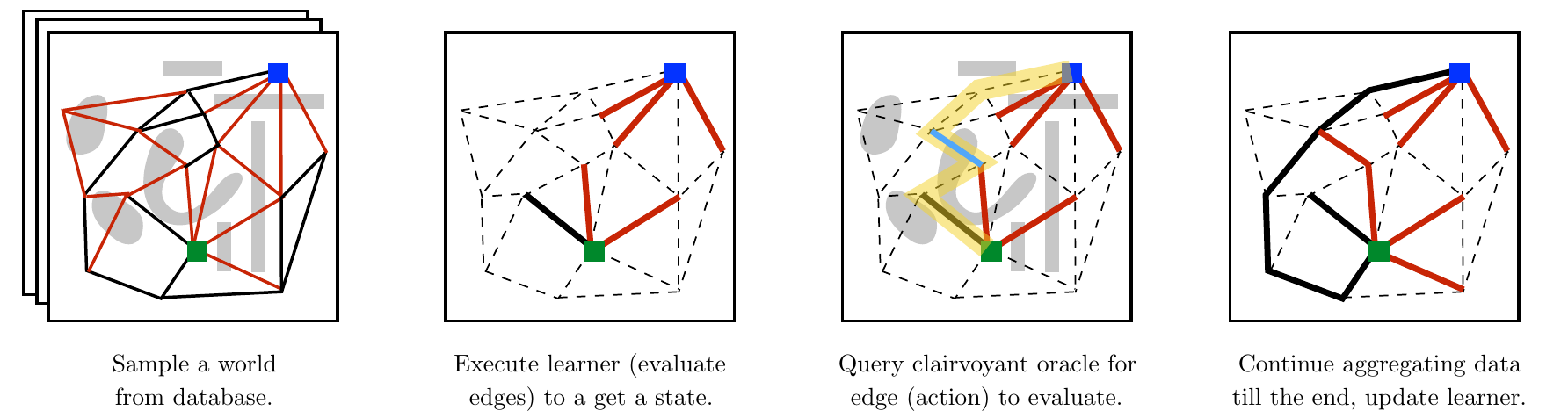}
    \caption{ Overview of \algName - a training procedure for a \selector to select edges to evaluate in the \lazysp framework. 
    In each training iteration, a world map $\world$ is sampled. The learner is executed upto a time step to get a state $s_t$ which is the set of edges evaluated and their outcomes. The learner has to decide which edge to evaluate on the current shortest path. It extracts features from every edge - we use a set of baseline heuristic values as features. The \selector asks a clairvoyant oracle selector (which has full knowledge of the world) which edge to evaluate. This is then added to a classification dataset and the learner is updated. This process is repeated over several iterations. \fullFigGap}
  \label{fig:illustration_algorithm}
\end{figure*}%

%% file: problem_formulation.tex
\section{Problem Formulation}
\label{sec:problem_formulation}

The overall objective is to design an algorithm that can solve the Shortest Path (SP) problem while minimizing the number of edges evaluated.

\subsection{The Shortest Path (SP) Problem}
\label{sec:problem_formulation:sp}

Let $\explicitGraph = \pair{\vertexSet}{\edgeSet}$ be an explicit graph where $\vertexSet$ denotes the set of vertices and $\edgeSet$ the set of edges. Given a start and goal vertex $\pair{\start}{\goal} \in \vertexSet$, a path $\path$ is represented as a sequence of vertices $\seq{\vertex}{l}$ such that $\vertex_1 = \start, \vertex_l = \goal, \forall i,~\pair{\vertex_i}{\vertex_{i+1}} \in \edgeSet$. 
We define a \emph{world} $\world: \edgeSet \rightarrow \{0,1\}$ as a mapping from edges to valid ($1$) or invalid ($0$). A path is said to be \emph{feasible} if all edges are valid, i.e. $\forall \edge \in \path, \world(\edge) = 1$. 
Let $\length: \edgeSet \rightarrow \real^+$ be the length of an edge. The length of a path is the sum of edge lengths, i.e. $\length(\path) = \sum_{\edge \in \path} \length(\edge)$. The objective of the SP problem is the find the shortest feasible path:
\begin{equation}
    \min_{\path} \; \length(\path) \suchthat{ \forall \edge\in\path, \world(\edge) = 1 }
\end{equation}
We now define a family of shortest path algorithms. Given a SP problem, the algorithms evaluate a set of edges $\evalEdges \subset \edgeSet$ (verify if they are valid) and return a path $\path^*$ upon halting. Two conditions must be met:
\begin{enumerate}
    \item The returned path $\path^*$ is verified to be feasible, i.e. $\forall \edge \in \path^*,~\edge \in \evalEdges,~\world(\edge) = 1$
    \item All paths shorter than $\path^*$ are verified to be infeasible, i.e. $\forall \path_i,~\length(\path_i) < \length(\path^*),~\exists \edge \in \path_i,~\edge \in \evalEdges,~\world(\edge) = 0 $
\end{enumerate}

\subsection{The Lazy Shortest Path (\lazysp) Framework}
\label{sec:problem_formulation:lazysp}
We are interested in shortest path algorithms that minimize the number of evaluated edges $\abs{\evalEdges}$.%
\footnote{The framework can be extended to handle non-uniform evaluation cost as well}
These are \emph{lazy} algorithms, i.e. they seek to defer the evaluation of an edge as much as possible. 
When this laziness is taken to the limit, one arrives at the \emph{Lazy Shortest Path} (\lazysp) class of algorithms. 
Under a set of assumptions, this framework can be shown to contain the optimally lazy algorithm~\citep{haghtalab2017provable}. 

\algoref{alg:lsp} describes the \lazysp framework. The algorithm maintains a set of evaluated edges that are valid $\edgesValid$ and invalid $\edgesInvalid$. At every iteration, the algorithm lazily finds the shortest path $\path$ on the potentially valid graph $\explicitGraph = \pair{\vertexSet}{\edgeSet \setminus \edgesInvalid}$ \emph{without evaluating any new edges} (\algolineref{alg:lsp:sp}). It then calls a function, \selector, to select an edge $\edge$ from this path $\path$ (\algolineref{alg:lsp:select}). Depending on the outcome, this edge is added to either $\edgesValid$ or $\edgesInvalid$. This process continues until the conditions in \sref{sec:problem_formulation:sp} are satisfied, i.e. the shortest feasible path is found. 

\begin{algorithm}[tb]
\SetAlgoLined
\caption{\lazysp \label{alg:lsp}}
\SetKwInOut{Input}{Input}
\SetKwInOut{Parameter}{Parameter}
\SetKwInOut{Output}{Output}
\Input{$\text{Graph } \explicitGraph, \text{ start } \start{}, \text{ goal } \goal{}, \text{ world } \world$}
\Parameter{$\selector$}
\Output{$\text{Path}\ \path^*, \text{ evaluated edges } \evalEdges{}$}
\vspace{2mm}
$\edgesValid \gets \emptyset$ \Comment{Valid evaluated edges} \\
$\edgesInvalid \gets \emptyset$ \Comment{Invaid evaluated edges} \\
\vspace{1mm}
\Repeat{feasible path found $\text{s.t.} \; \forall \edge \in \path, \edge \in \edgesValid$}
{
$\path \gets \textsc{ShortestPath}(\edgeSet \setminus \edgesInvalid)$ \label{alg:lsp:sp}\\
$\edge \gets \selector(\path, \edgesValid, \edgesInvalid)$ \Comment{Select edge on $\path$} \label{alg:lsp:select}\\
\uIf{$\world(\edge)\neq 0$}
{
    $\edgesValid \gets \edgesValid \cup \{ \edge \}$
}
\Else
{
    $\edgesInvalid \gets \edgesInvalid \cup \{ \edge \}$
}
}
\KwRet\ \{$\path^* \gets \path$, $\evalEdges \gets \edgesValid \cup \edgesInvalid$\};
\end{algorithm}

The algorithm has one free parameter - the \selector function. The only requirement for a valid \selector is to select an edge on the path. As shown in \citep{dellin2016unifying}, one can design a range of selectors such as:
\begin{enumerate}
    \item \selectorForward: select the first unevaluated edge $\edge \in \path$. Effective if invalid edges are near the start.
    \item \selectorBackward: select the last unevaluated edge $\edge \in \path$. Effective if invalid edges are near the goal.
    \item \selectorAlternate: alternates between first and last edge. This approach hedges its bets between start and goal.
    \item \selectorFailFast: selects the least likely edge $\edge \in \path$ to be valid based on prior data. 
    \item \selectorPostFailFast: selects the least likely edge $\edge \in \path$ to be valid using a Bayesian posterior based on edges checked so far.
\end{enumerate}

While these baseline selectors are very effective in practice, their performance, i.e. the number of edges evaluated $\abs{\evalEdges}$ depends on the underlying world $\world$ which dictates which edges are invalid. Hence the goal is to compute a good \selector that is effective given a \emph{distribution of worlds}, $P(\world)$. We formalize this as follows

\begin{problem}[Optimal Selector Problem] \label{prob:opt_select}
\\ Let the edges evaluated by \selector on world $\world$ be denoted by $\evalEdges(\world, \selector)$. Given a distribution of worlds, $P(\world)$, find a \selector that minimizes the expected number of evaluated edges, i.e. 
\begin{equation*}
\min \expect{\world \sim P(\world)}{\abs{\evalEdges(\world, \selector)}}
\end{equation*}
\end{problem}

Problem~\ref{prob:opt_select} is a sequential decision making problem, i.e. decisions made by the selector in one iteration (edge selected) affects the input to the selector in the next iteration (shortest path). We show how to formally handle this in the next section. 
It's interesting to note that Problem~\ref{prob:opt_select} can be solved optimally under certain strong assumptions as detailed in Section~\ref{sec:analysis}.

\subsection{Mapping the Optimal Selector Problem to an MDP}
\label{sec:problem_formulation:mdp}
We map Problem~\ref{prob:opt_select} to a Markov Decision Process (MDP) $\langle \stateSpace, \actionSpace, \transFn, \rewardFn, \discount \rangle$ as follows:

\subsubsection*{State Space}
The state $\state = (\edgesValid, \edgesInvalid)$ is the set of evaluated valid edges $\edgesValid$ and evaluated invalid edges $\edgesInvalid$. This can be represented by a vector of size $\abs{\edgeSet}$, each element being one of $\{-1, 0, 1\}$ - unevaluated, evaluated invalid, and evaluated valid respectively.
For simplicity, we assume that the explicit graph $\explicitGraph = (\vertexSet, \edgeSet)$ is fixed.\footnote{We can handle a varying graph by adding it to the state space.} 

Since each $\edge \in \edgeSet$ can be in one of $3$ sets, the cardinality of the state space is $\card{S} = 3^{\card{\edgeSet}}$.

The MDP has an absorbing goal state set $\stateAbs \subset \stateSpace$ which is a set of states where all the edges on the current shortest path are evaluated to be valid, i.e. 
\small
\begin{equation}
    \stateAbs = \setst{ ( \edgesValid, \edgesInvalid ) }{ \forall \edge \in \textsc{ShortestPath}(\edgeSet \setminus \edgesInvalid), \edge \in \edgesValid}
\end{equation}
\normalsize
 
\subsubsection*{Action Set}
The action set $\actionSpace(\state)$ is the set of unevaluated edges on the current shortest path, i.e.
\begin{equation}
    \actionSpace(\state) = \{ \edge \in \textsc{ShortestPath}(\edgeSet \setminus \edgesInvalid), \edge \notin \{ \edgesValid \cup \edgesInvalid \} \}
\end{equation}

\subsubsection*{Transition Function}
Given a world $\world$, the transition function is deterministic $\state' = \Gamma(\state, \action, \world)$:
\begin{equation}
    \Gamma(\state, \action, \world) =   \begin{cases} 
   (\edgesValid \cup \{\edge\}, \edgesInvalid)           & \text{if } \world(\edge) = 1 \\
   (\edgesValid, \edgesInvalid \cup \{ \edge \})         & \text{if } \world(\edge) = 0
  \end{cases}
\end{equation}

Since $\world$ is latent and distributed according to $P(\world)$, we have a stochastic transition function
\begin{equation*}
 \transFn(\state, \action, \state') = \sum_{\world} P(\world) \Ind(\state = \Gamma(\state, \action, \world))
\end{equation*}

\subsubsection*{Reward Function}
The reward function penalizes every state other than the absorbing goal state $\stateAbs$, i.e.
\begin{equation}
    \rewardFn(\state, \action) =   \begin{cases} 
   0             & \text{if } \state \in \stateAbs  \\
   -1            & \text{otherwise } 
  \end{cases}
\end{equation}

$\gamma$ is a discount factor that is used to favor immediate rewards over later ones. We can define the value of a given policy $\pi$ as $\valueFnPolicy{\pi}(s) = \expect{}{\sum_{t=0}^\infty\gamma^{t}c(s_t, a_t) \mid s_{0}=s}$ and the action-value function as $\QFn{\pi}(s,a) =$ \\ $\expect{}{\sum_{t=0}^\infty\gamma^{t}c(s_t, a_t) \mid s_{0}=s, a_{0}=a}$ where the expectation is over the policy and transition function. Further, we can also define the advantage function $A^{\pi}(s,a) = Q^{\pi}(s,a) - V^{\pi}(s)$ which measures how good an action is compared to the action taken by the policy in expectation.

%% file: challenge.tex

\section{Challenges in Solving the MDP}
\label{sec:challenge}

In this section, we examine tiny graphs and show that even for such problems, a choice of world distributions where edges are correlated can affect \selector choices. However, by solving the MDP using tabular Q-learning we can automatically recover the optimal \selector. 

\subsection{Experimental setup}
We train selectors on two different graphs and corresponding distribution of worlds $P(\world)$. 

\begin{figure}[!t]
\centering
\includegraphics[width=\columnwidth]{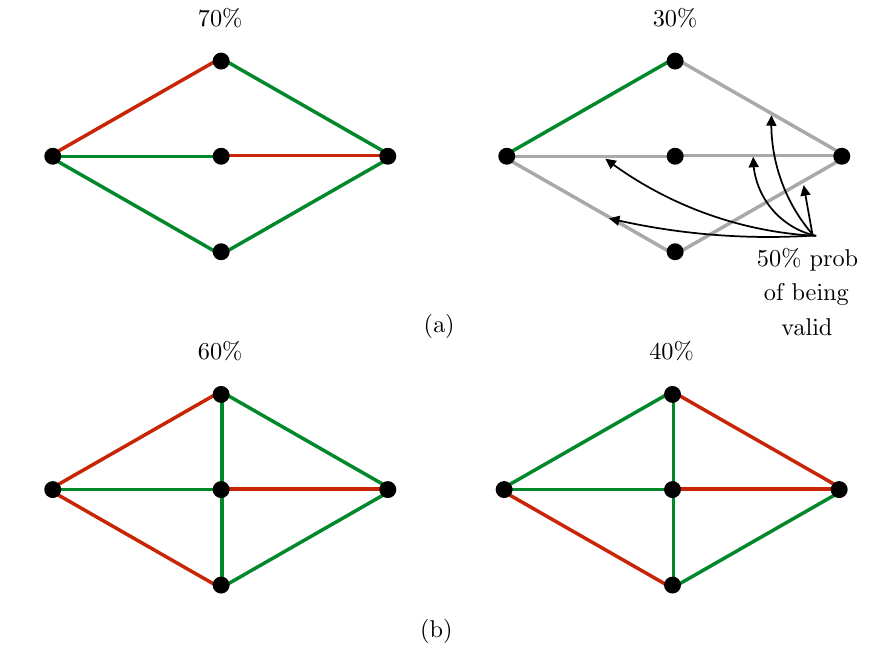}
\caption{Distribution over worlds for (a) Environment 1 and (b) Environment 2. The goal is to find a path from left to right. Edges are valid (green) or invalid (red) \fullFigGap}
\label{fig:simple_env}
\end{figure}

\subsubsection*{Environment 1} 
Fig.~\ref{fig:simple_env}(a) illustrates the distribution of Environment 1. The graph has $6$ edges. With $70 \%$ probability, $\mathtt{top\_left}$ edge is invalid. If $\mathtt{top\_left}$ is invalid, then $\mathtt{middle\_right}$ is always invalid. If $\mathtt{top\_left}$ is valid, then with $50\%$ probability, $\mathtt{top\_right}$ is invalid plus any one of remaining four are invalid. 

The optimal policy is to check $\mathtt{top\_left}$ edge first. 
\begin{itemize}
    \item[--] \emph{If invalid}, check $\mathtt{middle\_right}$ (which is necessarily invalid) and check bottom two edges which are feasible. This amounts to $4$ evaluated edges.
    \item[--] \emph{If valid}, check other edges in order as they all have $50\%$ probability of being valid.
\end{itemize}

\subsubsection*{Environment 2} 
Fig.~\ref{fig:simple_env}(b) illustrates the distribution of Environment 2. The graph has $8$ edges. 
$60 \%$ of the time $\mathtt{top\_left}$, $\mathtt{middle\_right}$ and $\mathtt{bottom\_left}$ are invalid. Else, the $\mathtt{top\_right}$ and the $\mathtt{middle\_right}$ are invalid. 
Intuitively, $60 \%$ of the time, \textsc{SelectAlternate} is optimal and $40 \%$ of the time, \textsc{SelectBackward} is the best. 

\subsection{Solving the MDP via Q-learning}

We apply tabular Q-learning~\cite{watkins1992q} to compute the optimal value $Q^*(\state, \action)$. Broadly speaking, the algorithm utilizes an $\epsilon-$greedy policy to visit states, gather rewards, and perform Bellman backups to update the value function. Environment 1 has $729$ states, Environment 2 has $6561$ states. The learning parameters are shown in Table~\ref{tab:q_learning_params}. 

\figref{fig:tab_q_results} shows the average reward during training for Q-learning. Environment 1 converges after $\approx 1000$ episodes, environment 2 after $\approx 3000$ episodes. Table~\ref{tab:experimental_results} shows a comparison of Q-learning with other heuristic baselines in terms of average reward on a validation dataset of $1000$ problems. In Environment 1, the learner discovers the optimal policy. Interestingly, \selectorAlternate also achieves this result since the correlated edges are alternating.  In Environment 2, the learner has a clear margin as compared to heuristic baselines, all of which are vulnerable to one of the modes.

This shows that, even on such small graphs, it is possible to create an environment where heuristic baselines fail. The fact that the learner can recover optimal policies is promising. 

\begin{table}[!t]
\centering
\caption{Q-learning parameters.}
\begin{tabulary}{\columnwidth}{LCC}\toprule
 {\bf Parameter} & {\bf Environment 1} & {\bf Environment 2} \\ \midrule
 Number of episodes & $3000$ & $3500$ \\
 Exploration episodes & $100$ & $150$ \\
 $\epsilon_0$  & $1$ & $1$ \\
 Discount factor & $1$ & $1$ \\
 Learning rate & $0.5$ & $0.5$ \\
 \bottomrule
\end{tabulary}
\label{tab:q_learning_params}
\end{table}

\begin{table}[t!]
\centering
\caption{Average reward after 1000 test episodes.}
\begin{tabulary}{\columnwidth}{LCC}\toprule
 {\bf Method}  & {\bf Environment 1} & {\bf Environment 2} \\ \midrule
 Tabular Q-learning  & $\cmark -3.85$ & $\cmark-5.24$\\
 \selectorForward & $-4.54$ & $-6.00$ \\
 \selectorBackward  & $-4.42$ & $-5.79$ \\
 \selectorAlternate & $-3.86$ & $-6.00$\\
 \selectorRandom & $-4.48$ & $-5.90$ \\
 \bottomrule
\end{tabulary}
\label{tab:experimental_results}
\end{table}

\begin{figure}[!t]
    \centering
    \captionsetup[subfigure]{justification=centering}
    \begin{subfigure}[t]{0.8\columnwidth}
        \includegraphics[height=1.0in]{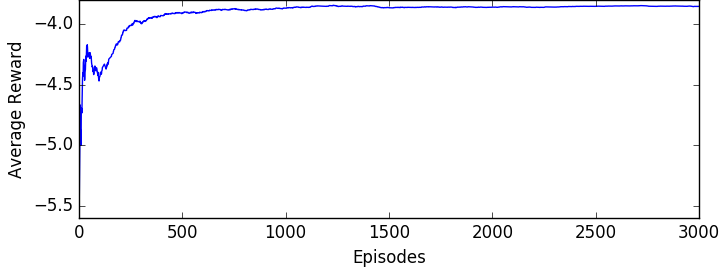}
        \caption{Environment 1 (3000 train episodes)}
    \end{subfigure}
   \hspace{15mm}
    \begin{subfigure}[t]{0.8\columnwidth}
        \includegraphics[height=1.0in]{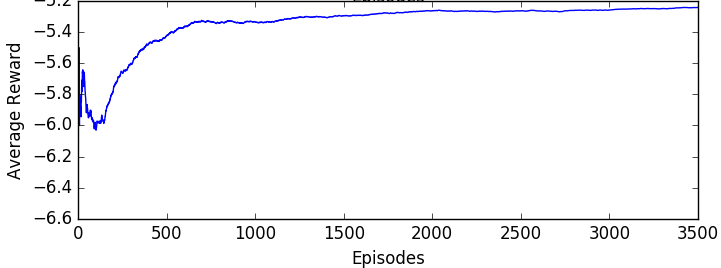}
        \caption{Environment 2 (3500 train episodes)}
    \end{subfigure}
    \caption{ Average reward per epsiode of Tabular Q-learning. \fullFigGap}
    \label{fig:tab_q_results}
\end{figure}

\subsection{Challenges on scaling to larger graphs}

While we can solve the MDP for tiny graphs, we run into a number of problems as we try to scale to larger graphs:

\paragraph{Exponentially large state space}
The size of the state space is $\card{S} = 3^{\card{\edgeSet}}$. This leads to exponentially slower convergence rates as the size of the graph increases. Even if we could manage to visit only the relevant portion of this space, this approach would not generalize across graphs. 

\paragraph{Convergence issues with approximate value iteration}
We can scale to large graphs if we use a function approximator. In this case, we have to featurize $(\state, \action)$ as a vector $f$, i.e. we are trying to approximate $Q(s,a) \approx Q(f)$. Fortunately, we have a set of baseline heuristics~\ref{sec:problem_formulation:lazysp} that can be used as a feature vector. This choice allows us to potentially improve upon baselines and easily switch between problem domains.

We run into another problem - approximate value iteration is not guaranteed to converge~\citep{gordon1995stable}. This is exaggerated in our case where $f$ is a set of baseline heuristics that may not retain the same information content as the state $\state$. Hence multiple states map to the same feature $f$, which leads to oscillations and local minima.

\paragraph{Sparse rewards}
Every state gets a penalization except the absorbing state, i.e. rewards are sparse. Because we are using a function approximator, updates to $Q(f)$ for reaching the goal state are overridden by updates due to $-1$ penalization.

%% file: approach.tex

\section{Approach}
\label{sec:approach}

Our approach, \algName (\algFullName), is to imitate clairvoyant oracles that can show how to evaluate edges optimally given full knowledge of the MDP at training time. To deal with distribution mismatch between oracle and learner, we use established techniques for iterative supervised learning.

\subsection{Optimistic Value Estimate using a Clairvoyant Oracle}
\label{sec:approach:oracle}
Consider the situation where the world $\world$ is fully known to the selector, i.e. the $0/1$ status of all edges are known. The selector can then judiciously select edges that are not only invalid, but eliminate paths quickly. We call such a selector a \emph{clairvoyant oracle}. We show that the optimal clairvoyant oracle, that evaluates the minimal number of edges, is the solution to a set cover problem.
\begin{theorem}[Clairvoyant Oracle as Set Cover]
\label{thm:set_cover}
Let $\state = (\edgesValid, \edgesInvalid)$ be a state. Let $V^*(\state, \world)$ be the optimal state action value when the world $\world$ is known. Then $V^*(\state, \world)$ is the solution to the following set cover problem
\begin{equation}
\begin{aligned}
\label{eq:set_cover}
-\minprob{\substack{\evalEdges \subset \setst{\edge \in \edgeSet}{\world(\edge)=0}} } & \abs{\evalEdges} \\
\mathrm{s.t.} \quad & \forall \path,\; \length(\path) < \length(\path^*), \path \cap \edgesInvalid = \emptyset, \\
		   & \hphantom{\forall \path,\;} \path \cap \evalEdges \neq \emptyset
\end{aligned}
\end{equation}
\end{theorem}
where $\path^*$ is the shortest feasible path for world $\world$.
\begin{proof}
(Sketch) Let $\pathSet = \{\path_1, \dots, \path_n\}$ be the set of paths that satisfy the constraints of~\eref{eq:set_cover}
\begin{enumerate}
	\item Shorter than $\path^*$, i.e. $\length(\path_i) < \length(\path^*)$
	\item Paths are not yet invalidated i.e. $\path \cap \edgesInvalid = \emptyset$
\end{enumerate}

Let $\setst{\edge \in \edgeSet}{\world(\edge)=0}$ be the set of invalid edges. Each edge $\edge$ covers a path $\path_i \in \pathSet$ if $\edge \in \path_i$. 
We define a cover as a set of edges $\evalEdges$ that covers all paths in $\pathSet$, i.e. $\path_i \cap \evalEdges \neq \emptyset$.

If we select a min cover, i.e. $\min \abs{\evalEdges}$ then all shorter paths will be eliminated. Hence this is equal to the optimal value $-V^*(\state, \world)$.
\end{proof}

Theorem~\ref{thm:set_cover} says that given a world and a state of the search, the clairvoyant oracle selects the minimum set of invalid edges to eliminate paths shorter than the shortest feasible path. 

Let $\policyOracle(\state,\world)$ be the corresponding oracle policy. We note that the optimal clairvoyant oracle can be used to derive an upper bound for the optimal value
\begin{equation}
\label{eq:upper_bound}
	Q^*(\state, \action) \leq  Q^{\policyOracle}(s,a) = \sum_{\world} P(\world | \state) Q^{\policyOracle}(s,a,\world)
\end{equation}
where $P(\world | \state)$ is the posterior distribution over worlds given state and $Q^{\policyOracle}(s,a,\world)$ is the value of executing action $a$ in state $s$ and subsequently rolling-out the oracle. Hence this upper bound can be used for learning. 

\subsection{Approximating the Clairvoyant Oracle}
\label{sec:approach:approx_oracle}

\begin{algorithm}[!t]
	\caption{\textsc{Approximate Clairvoyant Oracle}   \label{alg:approximate_oracle}}
	\SetKwInOut{Input}{Input}
\SetKwInOut{Output}{Output}
\Input{$\text{State } \state = (\edgesValid, \edgesInvalid), \text{ world } \world$}
\Output{$\text{Action } \action$}
	Compute shortest path $\hat{\path} = \textsc{ShortestPath}(\edgeSet \setminus \edgesInvalid)$ \\
	$\Delta \gets 0_{\abs{\edgeSet} \times 1}$\\
	\For{$\edge \in \hat{\path}, \world(\edge) = 0 $}
	{
		$\Delta(\edge) \gets \length(\textsc{ShortestPath}(\edgeSet \setminus \{ \edgesInvalid \cup \{ \edge \} \} )) - \length(\hat{\path})$\\
	}
	\KwRet\ Action $a = \argmaxprob{\edge \in \hat{\path}} \Delta(\edge)$;
\end{algorithm}

Since set cover is NP-Hard, we have to approximately solve \eref{eq:set_cover}. Fortunately, a greedy approximation exists which is near-optimal. The greedy algorithm iterates over the following rule:
\begin{equation}
\begin{aligned}
\label{eq:greedy_set_cover}
 	&\edge_i = \argmaxprob{\edge \in \edgeSet, \world(\edge)=0} \;\abs{ \setst{\path}{\length(\path) < \length(\path^*), \; \path \cap \edgesInvalid = \emptyset, \; \edge \in \path} } \\
 	&\evalEdges \gets \evalEdges \cup \{ \edge_i \}
\end{aligned}
\end{equation} 
The approach greedily selects an invalid edge that covers the maximum number of shorter paths, which have not yet been eliminated. This greedy process is repeated until all paths are eliminated.

There are two practical problems with computing such an oracle. First, exhaustively enumerating all shorter paths $\setst{\path}{\length(\path) < \length(\path^*)}$ is expensive, even at train time. Second, if we simply wish to query the oracle for which edge to select on the current shortest path $\hat{\path} = \textsc{ShortestPath}(\edgeSet \setminus \edgesInvalid)$, it has to execute \eref{eq:greedy_set_cover} potentially multiple times before such an edge is discovered - which also can be expensive. Hence we perform a double approximation.

The first approximation to \eref{eq:greedy_set_cover} is to constrain the oracle to only select an edge on the current shortest path $\hat{\path} = \textsc{ShortestPath}(\edgeSet \setminus \edgesInvalid)$ \vspace{-2mm}
\begin{equation}
\label{eq:constr_set_cover}
\approx \argmaxprob{\edge \in \hat{\path}, \; \world(\edge)=0} \;\abs{ \setst{\path}{\length(\path) \leq \length(\path^*), \; \path \cap \edgesInvalid = \emptyset, \; \edge \in \path} } 
\end{equation} 

The second approximation to \eref{eq:constr_set_cover} is to replace the number of paths covered with the marginal gain in path length on invalidating an edge. 
\begin{equation}
\label{eq:length_surrogate}
\approx \argmaxprob{\edge \in \hat{\path}, \; \world(\edge)=0} \; \length(\textsc{ShortestPath}(\edgeSet \setminus \{ \edgesInvalid \cup \{\edge\} \} )) - \length(\hat{\path})
\end{equation} 

Alg.~\ref{alg:approximate_oracle} summarizes this approximate clairvoyant oracle. 

\subsection{Bootstrapping with Imitation Learning}
\label{sec:approach:imitation}

Imitation learning is a principled way to use the clairvoyant oracle $\policyOracle(\state, \world)$ to assist in training the learner $\policy(\state)$. In our case, we can use the oracle action value $Q^{\policyOracle}(\state, \action)$ as a target for our learner as follows:
\begin{equation}
\label{eq:aggrevate}
\argmaxprob{\policy \in \policySpace}\;  \expect{\state \sim d_{\policy}(\state)}{Q^{\policyOracle}(\state, \policy(\state))} \\
\end{equation}
where $d_{\policy}(\state)$ is the distribution of states. Note that this is now a classification problem since the labels are provided by the oracle. However the distribution $d_{\policy}$ depends on the learner's $\policy$. \citet{ross2014reinforcement} show that this type of imitation learning problem can be reduced to interactive supervised learning.

We simplify further. Computing the oracle value requires rolling out the oracle until termination. We empirically found this to significantly slow down training time. Instead, we train the policy to directly predict the action that is selected by the oracle. This is the same as (\ref{eq:aggrevate}) but with a $0/1$ loss~\citep{ross2011reduction} -
\begin{equation}
\label{eq:dagger}
\argmaxprob{\policy \in \policySpace}\;  \expect{\state \sim d_{\policy}(\state)}{ \Ind(\policy(\state) = \policyOracle(\state, \world)) } \\
\end{equation}

We justify this simplification by first showing that maximizing action value is same as maximizing the advantage $Q^{\policyOracle}(\state, \action) - V^{\policyOracle}(\state)$. Since all the rewards are $-1$, the advantage can be lower bounded by the $0/1$ loss. We summarize this as follows:
\begin{equation}
\begin{aligned}
&\maxprob{\policy \in \policySpace}\;  \expect{\state \sim d_{\policy}(\state)}{Q^{\policyOracle}(\state, \policy(\state))} \\
= \quad & \maxprob{\policy \in \policySpace} \; \expect{\state \sim d_{\policy}(\state)}{Q^{\policyOracle}(\state, \policy(\state)) - V^{\policyOracle}(\state)} \\
\geq \quad & \maxprob{\policy \in \policySpace} \; \expect{\state \sim d_{\policy}(\state)}{ \Ind(\policy(\state) = \policyOracle(\state, \world)) - 1 } \\
\end{aligned}
\end{equation}

Finally, we do not use the exact clairvoyant oracle but rather an approximation (Section~\ref{sec:approach:approx_oracle}). In other words, there can exist policies $\policy \in \policySpace$ that outperform the oracle. In such a case, one can potentially apply policy improvement after imitation learning. However, we leave the exploration of this direction to future work.

\subsection{Algorithm}
\label{sec:approach:algorithm}

\begin{algorithm}[!t]
	\caption{\algName   \label{alg:lsp_learn}}
	\SetKwInOut{Input}{Input}
\SetKwInOut{Parameter}{Parameter}
\SetKwInOut{Output}{Output}
\Input{$\text{World distribution } P(\world), \text{ oracle } \policyOracle$}
\Parameter{$\text{Iter } N, \text{roll-in policy } \policyRoll, \text{ mixing } \{\mixParam_i\}_{i=1}^N$}
\Output{$\text{Policy } \policyLearn$}
	Initialize $\dataset \leftarrow \emptyset,\; \policyLearn_{1}$ to any policy in $\policySpace$ \\ \label{lst:line:}
	\For{$i = 1, \ldots, N$}
	{   Initialize sub-dataset $\dataset_{i} \leftarrow \emptyset$ \\
 		Let mixture policy be $\policyMix =\;\mixParam_{i}\policyRoll\; + \; (1 \; - \;\mixParam_{i})\policyLearn_{i} $ \\
 		\For{$j = 1,\ldots,m$}
 		{   
 			Sample $\world \sim P(\world)$; \\  
 			Rollin $\policyMix$ to get state trajectory $\{ \state_t \}_{t=1}^T$ \\
 			Invoke oracle to get $\action_t = \policyOracle(s_t, \world)$ \\
		    $\dataset_i \gets \dataset_{i} \cup \{ \left( \state_t, \action_t \right) \}_{t=1}^T$ \; 
 	    }
 	    Aggregate data $\dataset \leftarrow \dataset \cup \dataset_{i}$; \\
 	    Train classifier $\policyLearn_{i+1}$  on $\dataset$;\\  
	}
	\KwRet\ Best $\policyLearn$ on validation;
\end{algorithm}

The problem in (\ref{eq:dagger}) is a non-\emph{i.i.d} classification problem - the goal is to select the same action the oracle would select on the \emph{on policy distribution of learner}. \citet{ross2011reduction} proposed an algorithm, \daggerAlg, to exactly solve such problems. 

Alg.~\ref{alg:lsp_learn}, describes the $\algName$ framework which iteratively trains a sequence of policies $\seq{\policyLearn}{N}$. At every iteration $i$, we collect a dataset $\dataset_i$  by executing $m$ different episodes. In every episode, we sample a world $\world$ which already has every edge evaluated. We then roll-in a policy (execute a selector) which is a mixture $\policyMix$ that blends the learner's current policy, $\policyLearn_{i}$ and a base roll-in policy $\policyRoll$ using blending parameter $\mixParam_{i}$. At every time step $t$, we query the clairvoyant oracle with state $\state_t$ to receive an action $\action_t$. We use the approximate oracle in Alg.~\ref{alg:approximate_oracle}. We then extract a feature vector $f$ from all $(\state_t, \action)$ tuples and create a classification datapoint. We add this datapoint to the dataset $\dataset_i$. At the end of $m$ episodes, this data is then \emph{aggregated with the existing dataset} $\dataset$. A new classifier $\policyLearn_{i+1}$ is trained on the aggregated data. At the end of $N$ iterations, the algorithm returns the best performing policy on a set of held-out validation environments. 

We have two algorithms based on the choice of $\policyRoll$:
\begin{enumerate}
	\item \algName: We set $\policyRoll = \policyOracle$. This is the default mode of \daggerAlg. This uses the oracle state distribution to stabilize learning initially. 
	\item \algHeuristic: We set $\policyRoll$ to be the best performing heuristic on training as defined in \sref{sec:problem_formulation:lazysp}. This uses a heuristic state distribution to stabilize learning. Since the heuristic is realizable, it can have a stabilizing effect on datasets where the oracle is far from realizable.
\end{enumerate}

In the next section we discuss the different sources of error in our proposed framework and provide regret guarantees on the performance of \algName.

%% file: analysis.tex

\section{Theoretical Bounds on Performance}
\label{sec:analysis}

We wish to bound the performance of the policy output by \algName $\policyLEARN$ versus the optimal MDP policy $\policyOpt$.  

\begin{equation}
\expect{\state \sim d_{\policy}(\state)}{V^{\policyOpt}(s) - V^{\policyLEARN}(s)} \leq \epsilon
\end{equation}
where $V^{\policyOpt}(s)$ is the value of the optimal policy, $V^{\policyLEARN}(s)$ is the value of the STROLL policy and $d_{\policy}(\state)$ is the distribution of states visited by the learner. To understand the sub-optimality bound $\epsilon$, we need to examine the various components of this error. 

\subsection{Component 1: Unrealizability of the Clairvoyant Oracle}
 The first source of approximation error comes from our use of a clairvoyant oracle. The clairvoyant oracle has access to the true status of all edges in the graph whereas the learner is only privy to the status of edges checked so far. The realizability gap between the two is vast, resulting in a trivially large regret bound~\citep{ross2014reinforcement}. Instead, ~\citet{choudhury2017data} show that imitating the clairvoyant oracle is in fact equivalent to imitating a corresponding \emph{hallucinating oracle}, that computes an instantaneous posterior over worlds given the edge evaluations so far and computes the expected clairvoyant oracle action value over this posterior  i.e,
\begin{equation}
\label{eq:hallucinating_oracle}
Q^{OR}(\state, \action) =  \expect{\world \sim P(\world | \state)}{\QFn{\policyOR}(\state, \action, \world)}
\end{equation}
where $Q^{OR}(\state, \action)$ is the value of action $\action$ in state $\state$, $P(\world | \state)$ is the posterior over worlds and $\QFn{\policyOR}(\state, \action, \world)$ is the action value computed by the clairvoyant oracle. 

The hallucinating oracle policy is defined as one that greedily maximizes the action value $Q^{OR}(\state, \action)$
\begin{equation}
\policyORBel \cong \argmaxprob{a \in \actionSpace} \; Q^{OR}(\state, \action)
\end{equation}

The hallucinating oracle uses the same information as the learner and is equivalent to the well-known QMDP policy~\citep{Littman95learningpolicies}. The QMDP approximation, also known as \emph{hidsight optimization} operates under the assumption that the agent's uncertainty over the true world will be entirely eliminated after the next action. Thus, it simply estimates the expected Q-function weighted by the posterior probability over worlds. This results in an agent that chooses actions that maximize long-term rewards for all worlds weighted by their probability, while ignoring explicit information gathering. Policies based on this assumption have been shown to be effective in several POMDP domains in prior work~\citep{javdani2015shared,Koval-RSS-14,yoon2007ff}. However, the QMDP algorithm requires sampling from the \emph{true} posterior over worlds which is intractable in general. Nevertheless, Lemma 1 of \citep{choudhury2017data} state we are effectively imitating this hallucinating oracle by imitating the clairvoyant oracle. 

Hence, we assume that the error between the value of the optimal MDP policy and the hallucinating oracle is bounded by 
\begin{equation}
\label{eq:qmdp_error}
\norm{V^{\policyOpt}(s) - V^{\policyORBel}(s)}{\infty} \leq \epsilon_{\rm hal}
\end{equation}

We note that $\epsilon_{\rm hal}$ can be large for problems requiring a great deal of active information gathering and is hence difficult to quantify in general. 

\subsection{Component 2: Approximations in the Oracle Selector}
As discussed in Sec.~\ref{sec:approach:approx_oracle}, we make two simplifying approximations to efficiently compute the oracle at train time. Hence, instead of accurately computing $\QFn{\policyOR}(\state, \action, \world)$, an approximation $\QFn{\policyAOR}(\state, \action, \world)$ is computed instead. 

Consequently, this results in an approximate hallucinating oracle $\policyAORBel$ that computes an approximate action value $Q^{AOR}(\state, \action) =  \expect{\world \sim P(\world | \state)}{\QFn{\policyAOR}(\state, \action, \world)}$ and greedily maximizes it
\begin{equation}
\policyAORBel \cong \argmaxprob{a \in \actionSpace} \; Q^{AOR}(\state, \action)
\end{equation} 

We can then bound the error between value of the hallucinating oracle and the approximate hallucinating oracle by
\begin{equation}
\label{eq:approx_error}
\norm{V^{\policyORBel}(s) - V^{\policyAORBel}(s)}{\infty} \leq \epsilon_{\rm app}
\end{equation}

\subsection{Component 3: Errors from Imitation}

\algName imitates the actions demonstrated by the approximate oracle policy. Since our imitation learning back-end is \daggerAlg, we inherit the performance bounds from ~\citep{ross2011reduction}. 

The error between the learnt policy $\policyLearn$ and the demonstrator policy $\policyAORBel$ can be bounded using Theorem 4.1 in ~\citet{ross2011reduction}.
\begin{equation}
\begin{aligned}
\expect{\state \sim d_{\policy}(\state)}{V^{\policyAORBel}(s) - V^{\policyLEARN}(s)} &\leq  \epsilon_{\rm im} \\
& = A_{\infty} \epsilon_{\rm class} + \epsilon_{\rm reg} + O(\frac{1}{N})
\end{aligned}
\end{equation}
where $\epsilon_{\rm class}$ is the classification error of the best policy in the policy class on the aggregated dataset, $A_{\infty}$ is the maximum advantage w.r.t $\policyAORBel$ and $\epsilon_{\rm reg}$ is the average regret. 

We can now combine all components 
\begin{equation}
\begin{aligned}
& \expect{\state \sim d_{\policy}(\state)}{V^{\policyOpt}(s) - V^{\policyLEARN}(s)} \\
& \leq \expect{\state \sim d_{\policy}(\state)}{V^{\policyOpt}(s) - V^{\policyAORBel}(s)} + \epsilon_{\rm im} \\
& \leq \norm{V^{\policyOpt}(s) - V^{\policyAORBel}(s)}{\infty} + \epsilon_{\rm im} \\
& \leq \norm{V^{\policyOpt}(s) - V^{\policyORBel}(s)}{\infty} + \epsilon_{\rm app} + \epsilon_{\rm im} \\
& \leq \epsilon_{\rm hal} + \epsilon_{\rm app} + \epsilon_{\rm im} \\
\end{aligned}
\end{equation}

\subsection{Special Case: Optimal Selector for Independent Bernoulli Edges}
\label{sec:analysis:bern}
In this section we show that under certain conditions, \algName indeed contains the optimal policy in it's policy class.
We will show here that if $P(\world)$ is an independent Bernoulli distribution over edges $\vectorp$ and no two paths $\path_i, \path_j$ share an edge, the optimal \selector is the one that picks the edge on the shortest path with the lowest probability. As we show later in Section.~\ref{sec:experiments:setup}, this selector is part of the \algName policy class. Intuitively, the selector tries to eliminate each path as quickly as possible - the lack of overlap implies the selector does not have to reason over the consequences of eliminating a path.
{}

We first define a selector \selectorFailFast:
\begin{equation}
	\selectorFailFast \; \equiv \; \argmin_{\edge \in \path} \vectorp(\edge)
\end{equation}

We then show that \selectorFailFast eliminates a path optimally:
\begin{lemma}\label{lemma:fail_fast}
Given a path $\path$, \selectorFailFast minimizes the expected number of edges from $\path$ that are required to be evaluated to invalidate $\path$.
\end{lemma}

\begin{proof}
Given a path $\path$, a sequence of edges \\$S = \{e_1, e_2, \ldots, e_n\}$ belonging to the path,
and the corresponding priors of the edges being valid $(p_1, p_2, \ldots, p_n)$, 
let the expected number of edge evaluations to invalidate $\path$ be $\evalEdges(S)$ which is given by
\begin{equation}
\begin{aligned}
\expect{\vectorp}{ \evalEdges(S) } &= (1-p_1) + 2 p_1 (1-p_2) + \ldots \\
&= \sum_{l=1}^{n} \left( \prod_{m=1}^{l-1} p_m \right) \left(1-p_l\right)l 
\end{aligned}
\label{eqn:greedy_expected_edge_evaluations}
\end{equation}
Without loss of generality, let $p_i > p_{i+1}$ for a given $i$. Consider the alternate sequence of evaluations \\$S' = \{e_1, e_2, \ldots, e_{i+1}, e_i \ldots, e_n\}$ where the positions of the edges $e_i,~e_{i+1}$ are swapped. 
Consider the difference:
\begin{equation}
\begin{aligned}
& \expect{\vectorp}{ \evalEdges(S) } - \expect{\vectorp}{ \evalEdges(S') } \\
&= \ldots + \prod_{m=1}^{i-1} p_m \left[ (1-p_i) i + p_i(1-p_{i+1})(i+1) \right] + \ldots \\
&- \ldots + \prod_{m=1}^{i-1} p_m \left[ (1-p_{i+1}) i + p_{i+1}(1-p_{i})(i+1) \right] + \ldots \\
&= \prod_{m=1}^{i-1} p_m \left[-i(p_i - p_{i+1}) + (i+1)(p_i - p_{i+1}) \right] \\
&= \prod_{m=1}^{i-1} p_m (p_i - p_{i+1}) \\
& > 0
\end{aligned}
\label{eqn:greedy_expected_edge_evaluations}
\end{equation}
Since each such swap results in monotonic decrease in the objective, there exists an unique fixed point, i.e., the optimal sequence $S^*$ has $p_1 \leq p_2 \leq \ldots \leq p_n$.
\end{proof}

%% file: bayesian.tex

\section{The Bayesian Lazy Shortest Path Problem}
\label{sec:bayesian_sp}

In this section, we ask the question -- ``What are the minimal number of edge evaluations required to identify the shortest path under uncertainty?''. Formally, we define the Bayesian Lazy Shortest Path Problem \textemdash given a prior over world $P(\world)$, what is the minimal number of edges needed to be evaluated until we can certify with certainty that a given path is the shortest feasible path. There is an important distinction from the \lazysp paradigm \textemdash \emph{we do not need to evaluate every candidate shortest path in sequential order}. The prior $P(\world)$ maybe such that potential worlds $\world$ maybe ruled out by evaluating edges that do not necessarily lie on the shortest path. This can, in theory, give rise to algorithms that with very little evaluation collapse posterior over worlds such that a particular path $\xi$ can be claimed as the shortest feasible. 

We adopt a treatment similar to \citet{choudhury2018bayesian,choudhury2017active} where we show that the problem is equivalent to a problem in Bayesian Active Learning. While the problem is NP-Hard, we show that it is adaptive submodular which we leverage to derive greedy, near-optimal policies. \figref{fig:bayesian_lazy_sp} presents an overview of this connection. While the algorithm is simple to implement, it suffers from scalability as it requires explicit enumeration of all possible paths. Nonetheless, it serves as an important theoretical result and intuition pump for comparing efficacy of various \selector s in \lazysp.

\begin{figure*}[!t]
    \centering
    \includegraphics[width=0.9\textwidth]{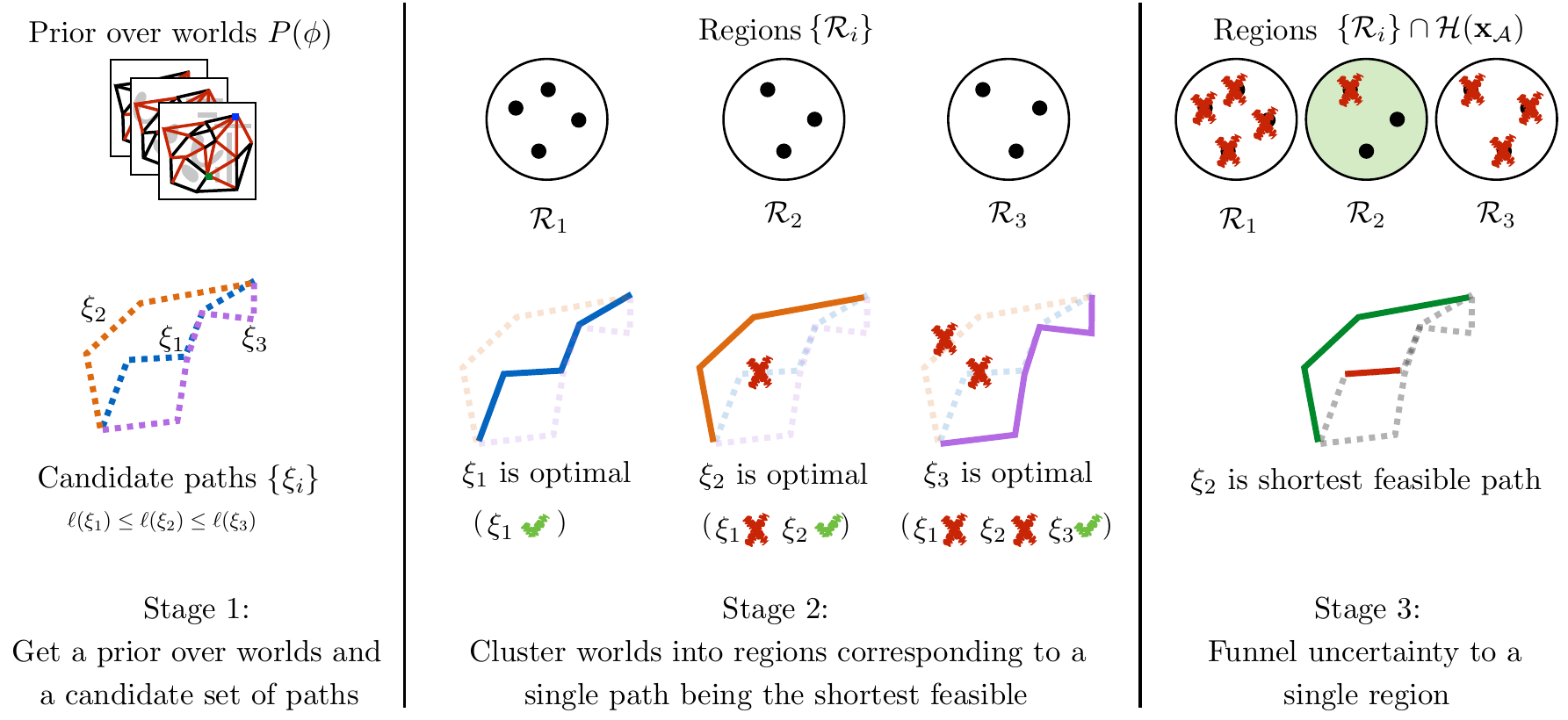}
    \caption{ Overview of mapping the Bayesian Lazy Shortest Path to Decision Region Determination. Stage 1 illustrates inputs to Problem~\ref{prob:blsp}. Stage 2 illustrates the mapping in Section~\ref{sec:drd_mapping}. Stage 3 illustrates the algorithm Algorithm~\ref{alg:bayesian_lsp}}
  \label{fig:bayesian_lazy_sp}
\end{figure*}%

\subsection{Preliminaries: Decision Region Determination (DRD)}

We first describe the problem of Bayesian Decision Region Determination (DRD). Given a set of hypotheses, a set of tests and a set of \emph{regions} (cluster of hypotheses), the goal is to perform a minimal set of tests to identify a region of hypotheses where the true hypothesis belongs\footnote{As opposed to the problem of optimal decision tree (ODT), where the true hypothesis must be identified.}. We formalize this below.

Let $\hypSet = \set{\hyp_1, \cdots, \hyp_\numHyp}$ be a set of candidate hypotheses, only one of which is true. We have a prior distribution $P(\hyp)$. Let $\testSet = \set{\test_1, \cdots, \test_\numTest}$ be a set of tests. Running a test $\test \in \testSet$ returns a binary outcome $\outcome_\test \in \{0, 1\}$ depending on the underlying hypothesis. Thus each hypothesis can be considered as a mapping from tests to outcomes $\hyp: \testSet \rightarrow \{0, 1\}$. The cost of performing a test is $\cost(\test)$\footnote{If we are only interested in minimizing the number of tests, then $\cost(\test)=1$ for all $\test$}. 

Let a region $\region \subseteq \hypSet$ be a set of a hypotheses. We will use $\set{\region_i}_{i=1}^{\numRegion}$ to denote the a set of regions. 

For a set of tests $\selTestSet \subseteq \testSet$ that are performed, let the observed outcome vector be denoted by $\obsOutcome$. Let the version space $\versionSpace(\obsOutcome)$ be the set of hypotheses consistent with outcome $\obsOutcome$, i.e. $\versionSpace(\obsOutcome) = \setst{\hyp \in \hypSet}{ \forall \test \in \selTestSet, \hyp(\test) = \obsOutcome(\test)}$. We assume that at least one hypothesis is true, i.e. $\abs{\versionSpace(\outcomeSet{\testSet})} \geq 1$. 

We define a policy $\policy$ as a mapping from the current outcome vector $\obsOutcome$ to the next test to select. A policy terminates when at least one region is valid, or all regions are invalid. Let $\hyp$ be the underlying hypothesis on which it is evaluated. Denote the outcome vector of a policy $\policy$ as $\obsOutcomeFunc{\policy}{\hyp}$. The expected cost of a policy $\policy$ is $\cost(\policy) = \expect{\hyp}{\cost(\obsOutcomeFunc{\policy}{\hyp}}$ where $\cost(\obsOutcome)$ is the cost of all tests $\test \in \selTestSet$. The objective is to compute a policy $\policyOpt$ with minimum cost such that all the uncertainty is funneled into one region, 
\begin{equation}
\label{eq:drd}
\policyOpt \in \argminprob{\policy} \;\cost(\policy) \; \mathrm{s.t} \; \forall \hyp , \exists \region_d \; : \; \versionSpace(\obsOutcome) \subseteq \region_d 
\end{equation}

Obtaining an approximate policy $\policy$ for which $\cost(\policy) \leq \cost(\policyOpt) . O(\log \numHyp)$ is NP-hard~\cite{chakaravarthy2007decision}. Fortunately, variants of DRD \cite{golovin2010near,javdani2014near} has been shown to have the property of adaptive submodularity~\cite{golovin2011adaptive}. This property implies that greedy policies are near-optimal. We harness this property in~\sref{sec:solution_drd}.

\subsection{Bayesian Lazy Shortest Path as a DRD problem}
\label{sec:drd_mapping}

We formally define the Bayesian Lazy Shortest Path problem, noting a distinction with the Optimal Selector problem Problem~\ref{prob:opt_select}. We are given a prior distribution of worlds, $P(\world)$. We are given a set of all candidate shortest paths $\{\path_i\}$. Let $\evalEdges(\world, \policy)$ denote the edges evaluated and outcomes by a policy $\policy$ on world $\world$. Let $\cost(\evalEdges(\world, \policy))$ denote the cost of edge evaluation. Let $P(\path_i \text{ is feasible} | \evalEdges(\world, \policy))$ be the posterior probability of a path being feasible. We then have the following problem:

\begin{problem}[Bayesian Lazy Shortest Path (SP)] \label{prob:blsp}
Minimize the cost of edge evaluation until a path is declared to be shortest feasible with probability $1$, i.e. it is deemed feasible and all shorter paths are deemed infeasible.
\begin{equation}
\begin{aligned}
  \min & \quad \expect{\world \sim P(\world)}{\cost{\evalEdges(\world, \policy)}} \\
  \text{s.t.} & \quad \exists \path^*: P(\path^* \text{ is feasible}| \evalEdges(\world, \policy)) = 1.0 \\
  & \quad \forall \path_i: ~\length(\path_i) \leq \length(\path^*),~P(\path_i \text{ is feasible}| \evalEdges(\world, \policy)) = 0.0
\end{aligned}
\end{equation}
\end{problem}

We now map this problem to Bayesian Decision Region Determination in Table~\ref{tab:equiv_problem}. Notably, each region consists of the number of worlds for which $\path_i$ is the shortest path. This mapping is also illustrated in \figref{fig:bayesian_lazy_sp}

\begin{table}[!htbp]
\centering
\caption{Mapping Bayesian Lazy Shortest Path to DRD}
\begin{tabulary}{\textwidth}{LCCCCCC}
\toprule
\textbf{Bayesian Lazy Shortest Path}	    & \textbf{Decision Region Determination} \\ \midrule
World ($\world$) & Hypothesis ($\hyp$) \\
Evaluate edge ($\edge$) & Execute test ($\test$) \\
Set of worlds for which path $\path_i$  & Region $\region_i \subset \hypSet$ \\
is shortest feasible & \\
Cost of evaluating edge ($\cost(\edge)$) & Cost of executing test ($\cost(\test)$) \\
\bottomrule
\end{tabulary}
\label{tab:equiv_problem}
\end{table}

\subsection{Efficienty solving DRD using the \ecsq algorithm}
\label{sec:solution_drd}

The DRD problem in \eref{eq:drd} has a special property -- \emph{the regions are disjoint}. \citet{golovin2010near} addresses this disjoint setting and propose a greedy algorithm, \ecsq, that is near-optimal. The idea behind \ecsq is elegant \textemdash define a graph with edges between hypotheses that we care to distinguish between. Tests `cut' edges inconsistent with outcomes, distinguishing between the two hypothesis. The aim is to cut all inconsistent edges with minimum expected incurred cost. We adopt this algorithm to solve the Bayesian Shortest Path Problem.

The \ecsq algorithm defines a graph $\graphDRD=(\vertexSetDRD, \edgeSetDRD)$, illustrated in \figref{fig:bayesian_lazy_sp}, where the nodes are hypotheses and edges are between hypotheses in \emph{different decision regions} $\edgeSetDRD = \cup_{i \neq j} \setst{ (\hyp, \hyp') }{\hyp \in \region_i, \hyp' \in \region_j}$. 

The weight of an edge is $\weight( (\hyp,\hyp')) = P(\hyp) P(\hyp')$. An edge is said to be `cut' by an observation if either hypothesis is inconsistent with the observation, i.e. weight is $0$. We can define the weight for all existing edges as a sum over individual weights $\weight(\edgeSetDRD) = \sum\limits_{\edge \in \edgeSetDRD} \weight(\edge)$. Notice that if we drive $\weight(\edgeSetDRD) \rightarrow 0$, we effectively disambiguate between hypothesis in different regions, i.e. we solve the DRD problem.  

We define $\weight(\{\region_i\})$ as the total weight of edges between regions. Since regions are disjoint, the total weight can be computed efficiently as 
\begin{equation}
\label{eq:weight}
\begin{aligned}
	\weight(\{\region_i\}) &= \frac{1}{2} \left( \sum\limits_{i \neq j} P(\region_i) P(\region_j) \right) \\ 
	&= \frac{1}{2} \left( ( \sum\limits_i P(\region_i) )^2 - \sum\limits_i P(\region_i)^2 \right)
\end{aligned}
\end{equation}

We can then define an objective function $\fec{\obsOutcome}$ that measures progress, i.e., how many edges have been cut by looking at the weight of the pruned regions to the original regions, i.e. 
\begin{equation}
\label{eq:ec2_fn}
\fec{\obsOutcome} = 1 - \frac{\weight(\{\region_i\} \cap \hypSet(\obsOutcome))}{\weight(\{\region_i\})}
\end{equation}

The objective function $\fec{\obsOutcome}$ is initially $0$ when no edges have been cut. When all the edges have been cut, i.e., $\weight(\{\region_i\} \cap \hypSet(\obsOutcome)) = 0$, the objective is $\fec{\obsOutcome}=1$. Our goal is to maximize this objective while incurring minimum cost of evaluating tests. We do so efficiently by observing that $\fec{}$ is \emph{adaptive submodular}:

\begin{lemma}
\label{lem:adapt_sub}
The objective function $\fec{}$ is strongly adaptive monotone and adaptive submodular
\end{lemma}
\begin{proof}
See proof of Proposition 2 in \citet{golovin2010near}
\end{proof}

The property above suggest a powerful technique \textemdash \emph{ greedily maximizing $\fec{}$ is near-optimal}~\citep{golovin2011adaptive}. \ecsq does just this. First note that given a test $\test$, we can compute the expected marginal gain of a test as 
\begin{equation}
	\gain{\test}{\obsOutcome} = \expect{\outcomeTest{\test}}{ \fec{ \obsOutcomeAdd{\test} } - \fec{\obsOutcome}}
\end{equation}

A greedy policy $\policyEC$ selects a test $\test^* \in \argmaxprob{\test} \frac{\gain{\test}{\obsOutcome}}{\cost(\test)}$. Expanding this we get:
\begin{equation}
\label{eq:greedy_policy}
\begin{aligned}
	&\argmaxprob{\test} \frac{1}{\cost(\test)} \expect{\outcomeTest{\test}}{ \fec{ \obsOutcomeAdd{\test} } - \fec{\obsOutcome}} \\
					   &= \frac{1}{\cost(\test)} \expect{\outcomeTest{\test}}{ \frac{\weight(\{\region_i\} \cap \hypSet(\obsOutcome)) - \weight(\{\region_i\} \cap \hypSet(\obsOutcomeAdd{\test}))}{\weight(\{\region_i\})} } \\
   					   &= \frac{1}{\cost(\test)} \expect{\outcomeTest{\test}}{ \weight(\{\region_i\} \cap \hypSet(\obsOutcome)) - \weight(\{\region_i\} \cap \hypSet(\obsOutcomeAdd{\test}))} \\
\end{aligned}	
\end{equation}
\begin{theorem}
\label{lem:near_opt}
The greedy policy $\policyEC$ is near-optimal, i.e.
\begin{equation}
	\cost(\policyEC) \leq (2 \log (1 / p_{\rm{min}})) \cost(\policyOpt)
\end{equation}
where $p_{\rm{min}} = \min_{\hyp \in \hypSet} P(\hyp)$ is the minimum prior probability for any hypothesis and $\policyOpt$ is the optimal policy for the DRD problem.
\end{theorem}
\begin{proof}
See proof of Theorem 3 in \citet{golovin2010near}
\end{proof}

\subsection{Algorithm}
We are now ready to apply the mapping in Section~\ref{sec:drd_mapping} to derive a greedy, near-optimal policy for evaluating edges until the shortest path is found. Algorithm.~\ref{alg:bayesian_lsp} describes the algorithm. It takes as input a set of candidate paths $\{\xi_i\}$and a prior over worlds $P(\world)$. It then clusters the worlds according to regions $\region_i$ such that all such worlds correspond to $\path_i$ to be the shortest feasible path. We define two lambda functions. First, computes a posterior over region $P(\region_i | \obsOutcome)$ given a set of edge evaluation outcomes $\obsOutcome$. Secondly, compute a weight function $\weight(\{\region_i\} \cap \hypSet(\obsOutcome))$ using \eqref{eq:weight}. This is the weight of all remaining edges that need to be cut. It then iteratively evaluates edges. In each iteration, it greedily selects an edge using \eqref{eq:greedy_policy}. The edge is evaluated, outcome is received and $\obsOutcome$ updated. If the uncertainty over worlds has been funneled into a single region $\region_i$, then we know for certain the shortest feasible path is $\xi_i$. The iteration terminates and this path is returned. 

\begin{algorithm}[tb]
\SetAlgoLined
\caption{Bayesian Lazy Shortest Path \label{alg:bayesian_lsp}}
\SetKwInOut{Input}{Input}
\SetKwInOut{Parameter}{Parameter}
\SetKwInOut{Output}{Output}
\Input{$\text{Set of candidate paths } \{\xi_i\}, \text{ prior over worlds } P(\world)$}
\Output{$\text{Shortest feasible path}\ \path^*$}
\vspace{2mm}
Create regions $\region_i = \setst{\world}{ \xi_i \text{ shortest feasible path in } \world}$  \\
Create function to compute posterior over regions $P(\region_i | \obsOutcome) \gets \sum_{\world \in \region_i} P(\world | \world \text { consistent with } \obsOutcome)$ \\
Create function to compute weights \eqref{eq:weight} $\weight(\{\region_i\} \cap \hypSet(\obsOutcome)) \gets \frac{1}{2} \left( ( \sum\limits_i P(\region_i | \obsOutcome) )^2 - \sum\limits_i P(\region_i | \obsOutcome)^2 \right)$ \\
Initialize list of (edges evaluated, outcome) $\obsOutcome \gets \emptyset$ \\
\vspace{1mm}
\Repeat{shortest feasible path found, i.e. $\exists \region_i, \hypSet(\obsOutcome) \subseteq \region_i$}
{
Select edge $\edge^*$ that maximizes expected marginal gain \eqref{eq:greedy_policy} $\argmaxprob{\edge}\; \expect{\outcomeTest{\edge}}{\frac{ \weight(\{\region_i\} \cap \hypSet(\obsOutcome)) - \weight(\{\region_i\} \cap \hypSet(\obsOutcomeAdd{\edge}))}{\cost(\edge)}}$
\\
Evaluate $\edge^*$ and observe outcome $\outcomeTest{\edge^*}$ \\
Add $\obsOutcome \gets \obsOutcome \cup (\edge^*, \outcomeTest{\edge^*})$
}
\KwRet\ \{$\path^* \gets \path_i$\};
\end{algorithm}

\subsection{Practical Challenges}

Practically, implementing Algorithm.~\ref{alg:bayesian_lsp} is difficult due to a number of challenges:
\begin{enumerate}
\item It requires enumerating the full set of candidate paths $\{\xi_i\}$ which can be $O(E!)$. 
\item At each iteration, the complexity is $O(\abs{\region}\abs{\world})$, i.e. the number of paths times the number of worlds, which makes computation quite expensive.
\item It assumes realizability of the world at test time.
\end{enumerate}

Depsite these shortcomings, it offers a very clean, analytic policy that could potentially be used if one were dealing with a small number of candidate paths. We leave exploring such algorithms for future work.

%% file: experiments.tex

\section{Experiments}
\label{sec:experiments}

\subsection{Experimental Setup}
\label{sec:experiments:setup}

We use datasets from \cite{choudhury2017active} in our experiments. The 2D datasets contain graphs with approximately 1600-5000 edges and varied obstacle distributions. The two 7D datasets involve a robot arm planning for a reaching task in clutter with large graphs containing 33286 edges.

\begin{figure*}[!t]
\centering
    \includegraphics[width=\textwidth]{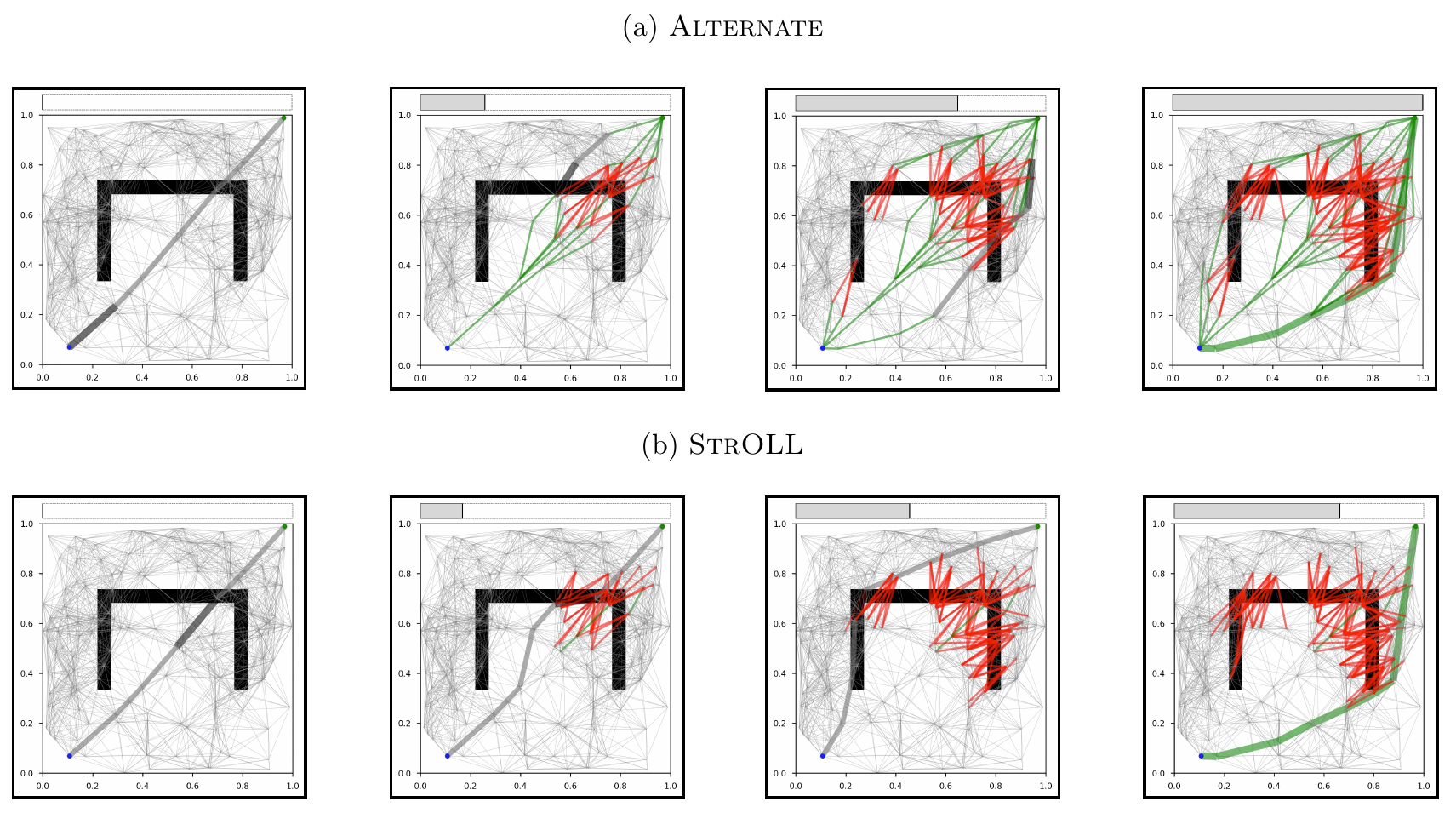}
\caption{ Snapshots of search progress using uniformed versus learned selectors on a \textsc{bugtrap} environment. The thicker gray edges depict the candidate shortest path with current selected edge in dark gray. Edges evaluated to be valid are shown in green and invalid edges are in red. The uninformed \selectorAlternate selector evaluates several valid edges thus wasting search effort whereas \algName focuses on edges more likely in collision thus leading to fewer overall edge evaluations as depicted in the top progress bar. \fullFigGap}
\label{fig:filmstrip_comparison} 
\end{figure*}

\begin{figure}[!t]
\centering
    \includegraphics[width=\columnwidth]{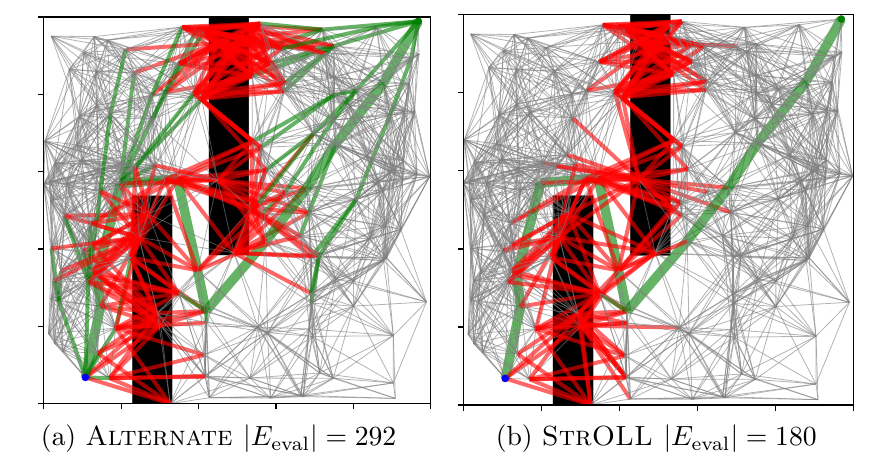}
\caption{Edges evaluated (green valid, red invalid) on a world from \textsc{Baffle}. (a) \selectorAlternate evaluates several valid edges (b) \algName evaluates many fewer edges, all of which are invalid and eliminate a large number of paths. \fullFigGap}
 \label{fig:edge_expand}
 \vspace{1mm}
\end{figure}

\subsubsection*{Learning Details}
We only consider policies that are a linear combination of a minimal set of features, where each feature is a different motion planning heuristic. The features we consider are: 
\begin{enumerate}
    \item \featPrior - the prior probability of an edge being invalid calculated over the training dataset.
    \item \featPost - the posterior probability of an edge being invalid given collision checks done thus far (described below)
    \item \featIndex - score ranging from 1 (first unchecked edge) to 0 (last unchecked edge).
    \item \featDeltaLength - hallucinate that an edge is invalid, then calculate the difference in length of new shortest path compared with the current shortest path.
    \item \featDeltaEval - hallucinate that an edge is invalid, the calculate the fraction of unevaluated edges on the new shortest path.
    \item \featPDL - calculated as \featPost $\times$ \featDeltaLength, it weighs the \featDeltaLength of an edge with the probability of it being invalid and is effective in practice (Table \ref{tab:benchmark_results}).
\end{enumerate}  

\textbf{\featPost Selector: }
We define the posterior selector used in a manner similar to $\selectorFailFast$ from~\sref{sec:analysis:bern} 
\begin{equation}
  \selectorPostFailFast \; \equiv \; \argmin_{\edge \in \path} \vectorp(\edge|\state_t)
\end{equation}
where $\state_t$ is the state of search at time t. We aproximate the posterior using the training dataset of $N$ worlds similar to \citep{choudhury2017active,choudhury2018bayesian} as follows - for each training world $\world_i$ a score $z_{i}$ is calculated based on the discrepancy between $s_t$ and the $s_{ti}$ where the latter is what the state of the search would be if agent were operating in $\world_i$, i.e
\begin{equation*}
z_{i} = -|s_t - s_{ti}|
\end{equation*}  
where, the difference follows directly from defintion in~\sref{sec:problem_formulation:mdp}. The probability for $\phi_i$ is then given by a softmax over training worlds,
\begin{equation*}
P(\world_i | \state_t) = \frac{e^{z_{i}}}{\sum_{\substack{k=0}}^{N} e^{z_k}}
\end{equation*}
Then, for every $\edge \in \path$,
\begin{equation}
\vectorp(\edge|\state_t) = \sum_{k=0}^{N} P(\world_k | \state_t) \world_i(\edge)
\end{equation}

\subsection{Baselines}
We compare our approach to common heuristics used in \lazysp as described in Section \ref{sec:problem_formulation:lazysp}. We also analyze the improvement in performance as compared to vanilla behavior cloning of the oracle and reinforcement learning from scratch. 

\subsection{Analysis of Overall Performance}
\begin{observation}
\algName has consistently strong performance across different datasets. 
\end{observation}
Table \ref{tab:benchmark_results} shows that \algName is able to learn policies competitive with other motion planning heuristics. No other heuristic has as consistent a performance across datasets.

\begin{observation}
The learner focuses collision checking on edges that are highly likely to be invalid and have a high measure of centrality.  
\end{observation} 
\figref{fig:weights_learner} shows the activation of different features across datasets. The learner places high importance on \featPost, \featDeltaLength and \featPDL. \featPost is an approximate likelihood of an edge being invalid and \featDeltaLength is an approximate measure of centrality i.e. edges with large \featDeltaLength have large number of paths passing through them (Note that the converse may not always apply).

\begin{observation}
On datasets with strong correlations among edges, heuristics that take obstacle distribution into account outperform uninformed heuristics, and \algName is able to learn significantly better policies than uninformed heuristics.
\end{observation}

Examples of such datasets are \textsc{gate},  \textsc{baffle}, \textsc{bugtrap} and \textsc{blob}. Here, \algName and \algHeuristic eliminate a large number of paths by only evaluating edges which are highly likely to be in collision and have several paths passing through them (Figs. \ref{fig:filmstrip_comparison}, \ref{fig:edge_expand}). In the 7D datasets, obstacles are highly concentrated near the goal region, which explains the strong performance of the uninformed \selectorBackward selector. However, due to a very large number of edges and limited training sets, \featPost and \featDeltaLength are inaccurate causing the learner to fail to outperform \selectorBackward.  

\begin{observation}
On datasets with uniformly spread obstacles, uninformed heuristics can perform better than \algName.
\end{observation}
Examples of such datasets are \textsc{twowall} and \textsc{forest} where the lack of structures makes features such as posterior uninformative. This combined with the non-realizability of the oracle makes it difficult for \algName to learn a strong policy.

\begin{figure}[!t]
\centering
    \includegraphics[width=\columnwidth]{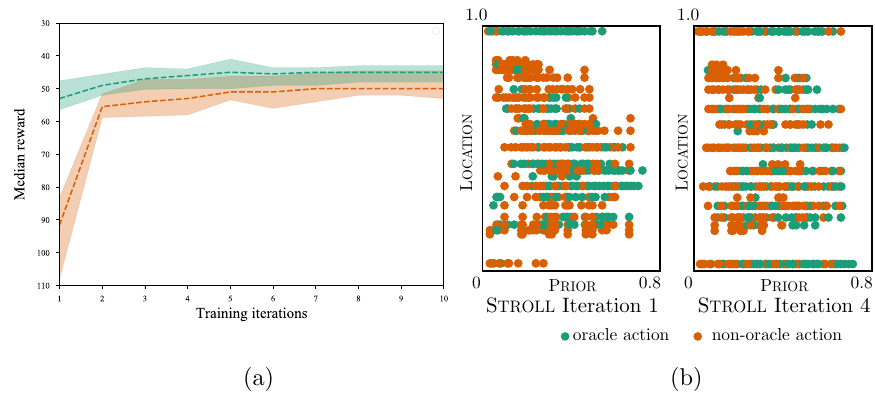}
\caption{ (a) \algHeuristic (green) vs \algName (orange) (b) Densification of data. }
    \label{fig:stroller_plots}
\end{figure}

\begin{figure}[!t]
    \centering
    \includegraphics[width=\columnwidth]{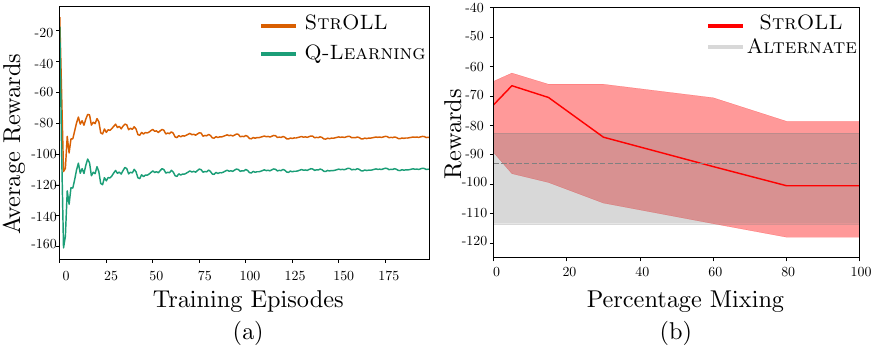}   
    \caption{(a) Running average reward for 200 episodes of training. Q learning suffers due to large state space and sparse rewards. (b) Performance on validation set of 200 worlds with contamination from different distribution. \vspace{-6mm}}
    \label{fig:q_learn_plot}
\end{figure}

\input{table}

\begin{figure*}[!t]
\centering
    \includegraphics[width=\textwidth]{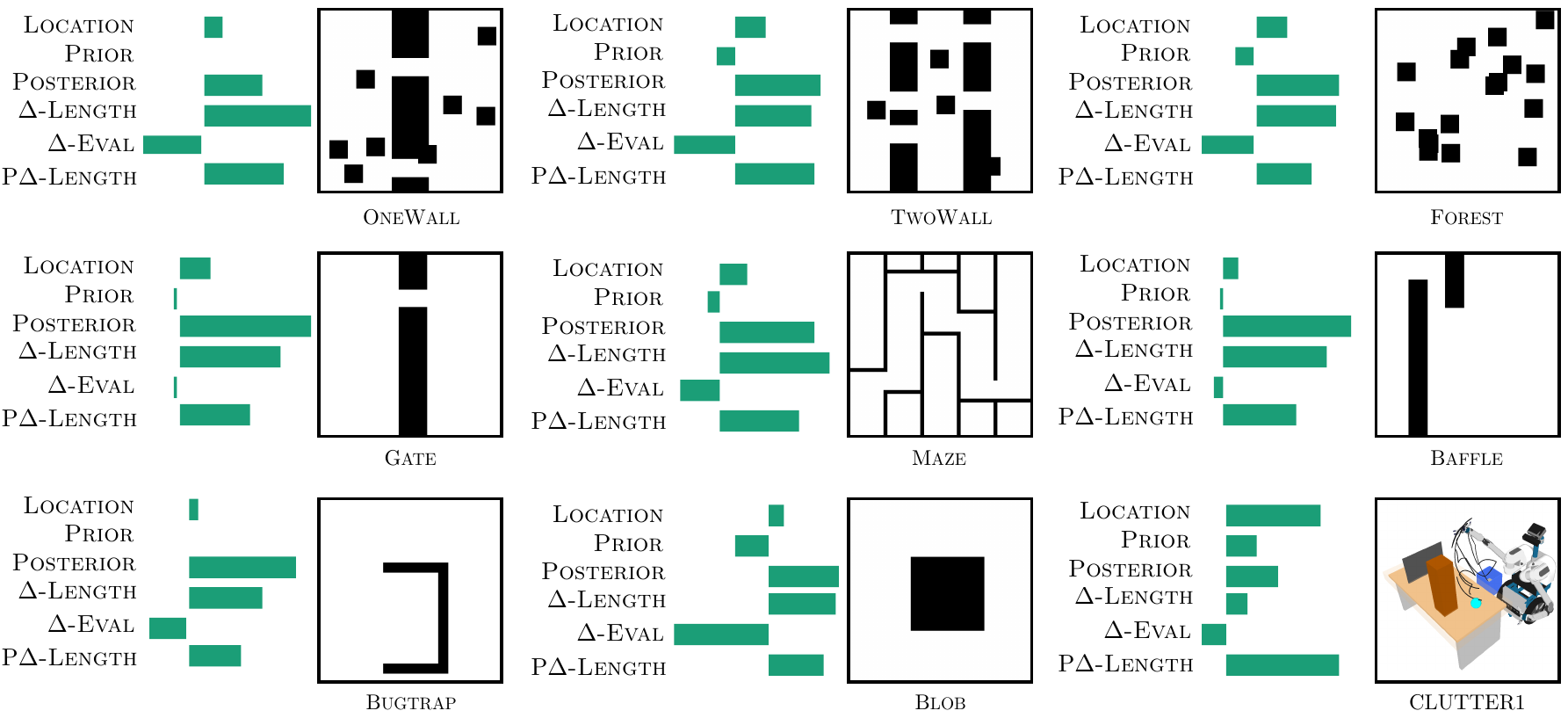}
\caption{ Weight bins depicting relative importance of each feature learned by the learner. \algName focuses on edges that are highly likely to be invalid and have high measure of centrality. \fullFigGap}
\label{fig:weights_learner} 
\end{figure*}

\subsection{Case Studies}
\begin{ques}
How does performance vary with training data?
\end{ques}
\figref{fig:stroller_plots}(a) shows the improvement in median validation reward with an increasing number of training iterations. Also, \figref{fig:stroller_plots}(b) shows that with more iterations, the learner visits diverse parts of the state-space on x-axis not visited by the oracle.  

\begin{ques}
How significant is the impact of heuristic roll-in on stabilizing learning in high-dimensional problems?
\end{ques}
\figref{fig:stroller_plots} shows a comparison of the median validation return per iteration using \algName versus \algHeuristic on \textsc{clutter1} dataset. Heuristic roll-in helps converge to a better policy in lesser number of iterations. Interestingly, the policy learned in the first iteration of \algHeuristic is significantly better than \algName, demonstrating the stabilizing effects of heuristic roll-in. 

\begin{ques}
How does performance compare to reinforcement learning with function approximation?
\end{ques}
\figref{fig:q_learn_plot}(a) shows training curves for \algName and \textsc{Q-Learning} with linear function approximation and experience replay. \algName is more sample efficient and converges to a competitive policy faster.

\begin{ques}
How does performance vary with train-test mismatch?
\end{ques}
\figref{fig:q_learn_plot}(b) shows a stress-test of a policy learned on \textsc{one wall} by running it on a validation set which is increasingly contaminated by environments from \textsc{forest}. The learned policy performs better than the best uninformed heuristic on \textsc{forest} for up to $60\%$ contamination.

%% file: table.tex
{
	\renewcommand{\arraystretch}{1.5}
\begin{table*}[!t]
\small
\centering
\caption{ Edges evaluated by different algorithms across different datasets (median, upper and lower C.I on 200 held-out environments). Highlighted is the best performing selector in terms of median score not counting the oracle.}
\begin{tabulary}{\textwidth}{LCCCCCCCCC}\toprule
       & {\bf \selectorOracle}  & {\bf \selectorBackward}    & {\bf \selectorAlternate} & {\bf \selectorFailFast} & {\bf \selectorPostFailFast} & {\bf \selectorPDL}  &{\bf \supervisedAlg} & {\bf \algName} & {\bf \algHeuristic} \\ \midrule
\multicolumn{10}{c}{ {\bf 2D Geometric Planning} }   \\
\textsc{OneWall}    & $80.0^{+6.0}_{-48.0}$  & $87.0^{+8.4}_{-41}$   & $112.0^{+12.8}_{-60.0}$ &  $82.0^{+3.0}_{-47.0}$ & $81.0^{+3.0}_{-49.0} $ & $85.0^{+6.8}_{-52.6.0} $ & \cmark $79.0^{+3.0}_{-45.0}$ & \cmark $79.0^{+5.0}_{-44.8}$ &\cmark $79.0^{+5.0}_{-44.8}$  \\ 
\textsc{TwoWall}    & $107.0^{+23.0}_{-0.0} $ & $199.0^{+8.0}_{-19.0}$  & $138.0^{+7.0}_{-2.0}$   & $178.0^{+0.0}_{-6.0}$ & $177.0^{+0.0}_{-7.0}$ & \cmark $120.0^{+19.0}_{-0.0} $ & $177.0^{+0.0}_{-6.0}$ & $177.0^{+0.0}_{-6.0}$ &  $170.0^{+12.2}_{-0.0}$  \\ 
\textsc{Forest}    & $90^{+14.4}_{-10.0}$  & $128.0^{+15.0}_{-16.4}$   & $115.0^{+12.0}_{-13.2}$ & $135.0^{+13.0}_{-16.0}$ & $116.0^{+13.2}_{-16.4}$ & \cmark $102.0^{+15.0}_{-12.0} $& $117.0^{+19.2}_{-17.0}$ & $115.0^{+20.0}_{-15.0}$  & $115.0^{+21.2}_{-13.4}$  \\ 
\textsc{Gate}    & $50.0^{+6.0}_{-9.0}$ & $74.0^{+8.0}_{-9.0}$   & $75.0^{+14.2}_{-6.2}$ & $60.0^{+8.0}_{-6.2}$ & $50.0^{+7.0}_{-8.2}$ & $53.0^{+7.0}_{-9.2} $ & $50.0^{+7.0}_{-7.2}$ & \cmark$48.0^{+10.0}_{-7.2}$ & \cmark$48.0^{+9.2}_{-9.2}$  \\ 
\textsc{Maze}    & $537.0^{+37.0}_{-24.6}$ & $668.5^{+40.3}_{-56.1}$   &  $613.0^{+39.6}_{-33}$ & $512^{+52.2}_{-34.0}$ & $516.5^{+33.70}_{-36.50}$ & $529.0^{+40.0}_{-37.2} $ & \cmark $502.5^{+58.7}_{-28.2}$ & $512.0^{+42.0}_{-31.0}$ & $554.0^{+52.2}_{-59.0}$  \\ 
\textsc{Baffle}   & $219.0^{+18.0}_{-12.0}$ & $244.0^{+14.6}_{-6.0}$   & $311.0^{+8.0}_{-15.6}$ & $232.0^{+6.0}_{-12.0}$ & $211.0^{+7.8}_{-6.0}$ & $230.0^{+18.0}_{-17.0} $ &\cmark $206.0^{+6.8}_{-3.0}$ &\cmark $205.0^{+6.0}_{-3.0}$ &$207.0^{+7.0}_{-2.0} $  \\ 
\textsc{Bugtrap}  & $77.0^{+12.0}_{-9.4}$  & $104.0^{+6.0}_{-14.0}$  & $112.5^{+16.9}_{-11.5}$ & $90.5^{+10.9}_{-13.5}$ & $75.5^{+16.5}_{-6.5}$ & $84.5^{+12.9}_{-10.5} $ & \cmark $75.0^{+15.4}_{-6.4}$ & \cmark $75.0^{+15.4}_{-6.4}$ & \cmark $75.0^{+15.4}_{-6.4}$  \\ 
\textsc{Blob}  & $72.0^{+12.0}_{-4.0}$ & $92.0^{+5.4}_{-3.4}$   & $109.0^{+5.0}_{-7.0}$ &\cmark $70.0^{+6.0}_{-6.0}$ &\cmark $70.0^{+6.0}_{-6.0}$ & $80.0^{+9.0}_{-3.0} $ & $72.0^{+5.8}_{8.0}$ &\cmark $70.0^{+6.0}_{-6.0}$  &\cmark $70.0^{+6.0}_{-6.0}$  \\ 
\multicolumn{10}{c}{ {\bf 7D Manipulation Planning} }   \\
\textsc{Clutter1}  & $35.5^{+1.5}_{-1.5}$  & $38.0^{+12.0}_{-10.0}$   & $44.0^{+14.2}_{-4.0}$ & $92.0^{+5.0}_{-0.0}$ & $88.0^{+9.6}_{-0.0}$ & \cmark $37.5^{+2.1}_{-1.5} $ & $95.5^{+17.7}_{-11.1}$ & $50.0^{+3.0}_{-6.0}$ & $45.0^{+2.0}_{-1.6}$   \\ 
\textsc{Clutter2}  & $34.0^{+1.0}_{-2.0}$  & \cmark $32.0^{+3.0}_{-4.0}$   & $41.0^{+2.0}_{-2.2}$ & $85.0^{+0.0}_{-3.0}$ & $84.0^{+0.0}_{-2.0}$ & $37.0^{+0.0}_{-5.0} $ & $104.0^{+6.2}_{-6.8}$ & $47.0^{+2.0}_{-11.2}$ & $ 44.0^{+5.0}_{-7.2}$ \\ 
\bottomrule
\end{tabulary}
\label{tab:benchmark_results}
\end{table*}
}

%% file: related_work.tex

\section{Related Work}
\label{sec:related_work}

In domains where edge evaluations are expensive and dominate planning time, a \emph{lazy approach} is often employed~\citep{bohlin2000path} wherein the graph is constructed \emph{without} testing if edges are collision-free. \lazysp~\citep{dellin2016unifying} extends the graph all the way to the goal, before evaluating edges. LWA*~\cite{cohen2015planning} extends the graph only a single step before evaluation. (LRA*)~\citep{Mandalika18} is able to trade-off between them by allowing the search to go to an arbitrary lookahead. The principle of laziness is reflected in similar techniques for randomized search~\citep{gammell2015batch,hauser2015lazy}. 

Several previous works investigated leveraging priors in search. FuzzyPRM~\citep{nielsen2000two} evaluates paths that minimize the probability of collision. The Anytime Edge Evaluation (AEE*) framework~\cite{narayanan2017heuristic} uses an anytime strategy for edge evaluation informed by priors. \bisect~\cite{choudhury2017active} and \direct~\cite{choudhury2018bayesian} casts search as Bayesian active learning to derive edge evaluation policies. However, none of these approaches formalized the problem of minimizing edge evaluation till the shortest feasible path is found. Our paper formalizes this problem, examines it in both the Bayesian setting as well as offers a practical learning based approach by leveraging imitation learning.

Efficient collision checking has its own history in the context of motion planning. Other approaches model belief over the configuration space to speed-up collision checking \citep{huh2016learning,choudhury2016pareto}, sample vertices in promising regions \citep{bialkowski2013free} or grow the search tree to explore the configuration space \citep{hsu1997path,burns2005sampling,lacevic2016burs}. However, these approaches make geometric assumptions and rely on domain knowledge. We work directly with graphs and are agnostic with respect to the domain.

Several recent works use imitation learning~\citep{ross2011reduction,ross2014reinforcement,sun2017deeply} to bootstrap reinforcement learning. THOR~\citep{sun2018truncated} performs a multi-step search to gain advantage over the reference policy. LOKI~\citep{cheng2018fast} switches from IL to RL. Imitation of clairvoyant oracles has been used in multiple domains like information gathering~\citep{choudhury2017data}, heuristic search~\citep{bhardwaj2017heuristic}, and MPC~\citep{kahn2016plato,tamar2016hindsight}.

%% file: discussions.tex

\section{Discussion}
\label{sec:discussion}

We examined the problem of minimizing edge evaluations in lazy search on a distribution of worlds. We first formulated the problem of deciding which edge to evaluate as an MDP and
presented an algorithm to learn policies by imitating clairvoyant oracles, which, if the world is known, can optimally evaluate edges. Further, we provide a theoretical analysis of our proposed framework that details different sources of approximation error. We also analyze the problem in the Bayesian setting and draw a novel connection to Bayesian Active Learning.
However, the current approach has certain limitations. First, we only consider learning edge selector policies that are a linear combination of heuristic edge selectors. Although we found using baseline heuristics as features to yield strong results, the limited representational power of hand designed features can affect the performance of the policy. This can be improved in the future by considering more general function approximators such as Graph Neural Networks~\citep{wu2020comprehensive} to represent the selector policy. 
Second, while imitation learning of clairvoyant oracles is effective, the approach may be further improved through reinforcement learning~\citep{sun2018truncated,cheng2018fast} since in in practice we do not use the exact oracle but a sub-optimal approximation which means that errors in the oracle will transfer to the learner, limiting performance.